\documentclass[journal]{IEEEtran}
\usepackage{graphicx} 
\usepackage{lipsum}
\usepackage[T1]{fontenc}
\usepackage{amsmath}
\usepackage{booktabs}
\usepackage{booktabs}
\usepackage{array}

\usepackage{amssymb}
\usepackage{amsthm}
\usepackage{stackengine}
\usepackage{lipsum}
\usepackage{multicol}

\usepackage[usenames, dvipsnames]{color}
\usepackage{caption}
\usepackage{subcaption}
\newtheorem{theorem}{Theorem}

\newtheorem{remark}{Remark} 
\newtheorem{lem}{Lemma}
\newtheorem{assum}{Assumption}
\newtheorem{corollary}{Corollary}
\newtheorem{propos}{Proposition}

\usepackage{lipsum}
\usepackage{mathtools}
\usepackage{cuted}
\usepackage{algorithm,algorithmic}
\usepackage{epstopdf}
\usepackage{cite}
\DeclareMathOperator*{\argmin}{argmin}
\usepackage{amsmath,amssymb,amsfonts}
\usepackage{algorithmic}
\usepackage{textcomp}
\begin{document}
\title{Adaptive Lattice-based Motion Planning}
\author{
Abhishek Dhar, Sarthak Mishra, Spandan Roy and Daniel Axehill
\thanks{This work was supported by the excellence center ELLIIT and by Strategic Vehicle Research and Innovation, FFI, within the project iDecide.}
\thanks{Abhishek Dhar is with Technology and Innovation, Epiroc, Sweden. e-mail:
      ({\tt\small abhishek.dharr@gmail.com})}
\thanks{Sarthak Mishra and Spandan Roy are with the Robotics Research Center, International Institute of Information Technology, Hyderabad, India. e-mail:
       ({\tt\small  sarthak.mishra@research.iiit.ac.in, spandan.roy@iiit.ac.in})}
\thanks{Daniel Axehill is with the Division of Automatic Control, Department of Electrical Engineering, Linköping University, Sweden. e-mail:
       ({\tt\small daniel.axehill@liu.se})}
}

\maketitle
\begin{abstract}
This paper proposes an adaptive lattice-based motion planning solution to address the problem of generating feasible trajectories for systems, represented by a linearly parameterizable non-linear model operating within a cluttered environment. The system model is considered to have uncertain model parameters. The key idea here is to utilize input/output data online to update the model set containing the uncertain system parameter, as well as a dynamic estimated parameter of the model, so that the associated model estimation error reduces over time. This in turn improves the quality of the motion primitives generated by the lattice-based motion planner using a nominal estimated model selected on the basis of suitable criteria. The motion primitives are also equipped with tubes to account for the model mismatch between the nominal estimated model and the true system model, to guarantee collision-free overall motion. The tubes are of uniform size, which is directly proportional to the size of the model set containing the uncertain system parameter. The adaptive learning module guarantees a reduction in the diameter of the model set as well as in the parameter estimation error between the dynamic estimated parameter and the true system parameter. This directly implies a reduction in the size of the implemented tubes and guarantees that the utilized motion primitives go arbitrarily close to the resolution-optimal motion primitives associated with the true model of the system, thus significantly improving the overall motion planning performance over time. The efficiency of the motion planner is demonstrated by a suitable simulation example that considers a drone model represented by Euler-Lagrange dynamics containing uncertain parameters and operating within a cluttered environment. 
\end{abstract}
\section{Introduction}
Motion planning is a fundamental problem in robotics and autonomous systems, concerned with computing a sequence of feasible movements that a robot or agent can follow to reach a goal state from a given starting point while avoiding obstacles and satisfying physical constraints \cite{paden2016}. It plays a crucial role in enabling intelligent behavior in various applications such as robotic manipulation \cite{billard2019}, autonomous driving \cite{yurtsever2020}, drone navigation \cite{lu2018}, and mobile robotics \cite{siegwart2011}. Classical motion planning techniques include algorithms such as Rapidly-exploring Random Trees (RRT) \cite{karaman2011}, Probabilistic Roadmaps (PRM) \cite{kavraki1998}, and A* search \cite{erke2020}, which explore the robot’s configuration space to compute collision-free paths. These methods typically aim to balance computational efficiency with completeness, striving to find feasible solutions quickly, though not always guaranteeing optimality. Furthermore, these methods are not inherently capable of handling uncertainties in system model or the environment. To ensure reliable and robust operation under uncertain conditions, appropriate modifications or extensions are necessary within the design. \par
To address the challenge of handling uncertainties, robust and learning-based motion planning has emerged as a powerful alternative to classical planning techniques, offering improved performance by using data-driven models. Numerous studies have explored its potential across a wide range of applications. For example, \cite{faust2018} proposes PRM-RL, a hybrid framework that combines Probabilistic Roadmaps with reinforcement learning to enable efficient long-range navigation in challenging environments. Similarly, \cite{pfeiffer2017} introduces an end-to-end learning-based navigation system using deep convolution neural networks, capable of mapping sensor data directly to motion commands for navigation in cluttered indoor spaces. The article \cite{zhang2025} demonstrates the use of imitation learning with reinforcement learning for agile quadrotor flight through complex obstacle fields, showing that neural networks can learn aggressive maneuvers from expert demonstrations. The authors in \cite{ichter2018} developed a learning-based sampling distribution method that accelerates sampling-based planners by guiding exploration toward promising regions of the configuration space. Similarly, \cite{codevilla2018} explores conditional imitation learning for autonomous driving, where the planner adapts its behavior based on high-level commands. The works in \cite{manjunath,majumdar,gurgen} propose robust motion planning solutions for nonlinear systems under uncertainty highlighting various novelties -  use of control barrier functions along with the classical RRT strategy, employing funnel libraries to manage real-time disturbances and generating safety certificates for trajectories of nonlinear systems subject to additive disturbances, etc. The authors of \cite{tsukamoto}, in their work propose a learning-based robust motion planner that leverages contraction theory to ensure stability. The robust lattice-based motion planner proposed in \cite{dhar2023} utilizes all available uncertainty information to generate a safe motion plan, which delivers comparatively superior performance with respect to the strategies exploiting just the worst case knowledge of the uncertainty.\par
Despite these advancements, both robust as well as learning-based motion planning still faces significant challenges that hinder its widespread deployment in safety-critical and real-world applications. For learning-based motion planning solutions, one key limitation is the dependence on large, diverse, and well-annotated datasets for training; the performance of these systems often degrades significantly when exposed to out-of-distribution scenarios or environmental changes that were not represented during training. Additionally, many learning-based methods do not provide formal guarantees of safety, completeness, or optimality, which are crucial in applications such as autonomous driving and aerial robotics. This lack of theoretical guarantees can lead to unpredictable behavior in corner cases or rare events. Furthermore, training such models can be computationally intensive, and their generalization capability is often limited, requiring extensive retraining or domain adaptation when deployed in new environments. In case of robust motion planning solutions, in most cases a safe motion plan is ensured by using knowledge of worst-case of uncertainty, and they generally are unable to update information about uncertainties and incorporate them in the planning process. In these approaches improvements in planning performance are typically achievable only if the worst-case disturbance bounds are reduced, which hints at the systematic fusion of robust and learning-based motion planning strategies. This motivates to develop hybrid approaches that combine the robustness with learning capabilities with the aim to achieve structured reliability while guaranteeing system theoretic properties.\par
The aforementioned drawbacks are mitigated in this paper, where an adaptive lattice-based motion planner is proposed for systems represented by  a nonlinear model with parametric uncertainties. Lattice-based motion planners \cite{bergman,bergman1,ljungqvist} offer a more structured and deterministic approach by discretizing the state space into a grid or lattice and generates resolution-optimal motion plan. Each node in the lattice represents a reachable state, and edges correspond to precomputed, dynamically feasible motion primitives, which are generated by solving a suitably constructed two-point boundary-value optimal control problem. This framework is particularly advantageous for systems with complex dynamics or non-holonomic constraints, where random sampling methods may struggle to generate feasible or smooth trajectories. Despite the strong performance and structured nature of lattice-based motion planning strategies, these approaches do not inherently account for uncertainties, which poses a significant limitation in real-world applications. The proposed adaptive motion-planning strategy analytically guarantees generation of safe, feasible and resolution-optimal trajectories for uncertain nonlinear systems, by leveraging adaptive identification strategy \cite{landau,ioannou2006adaptive}, which improves the quality of available uncertainty knowledge online over time. The proposed strategy also incorporates robustness in the design by having suitable margin to obstacles, such that the planner generates collision free motion when learning is ongoing. It is intuitively clear that the gradual improvement in the uncertainty knowledge leads to reduction in required margin to obstacle and at the same time improvement in the quality of motion primitives in terms of optimality. \par
\noindent \textit{Summary of contributions:}\\
This paper systematically combines an adaptive identification strategy with a classical lattice-based motion planner. In this approach, a library of motion primitives is generated offline using multiple nominal model parameter estimates, which are static and structurally correlated to the true uncertain system model parameters. To address for the error in the generated motion primitives, arising due to the model mismatch between the nominal parameter estimates and the true model parameters, the motion primitives are equipped with tubes of appropriate diameter, proportional to the size of the set containing the uncertain parameters in the parameter space. A dynamic estimated model is also considered, whose parameters are updated online. A gradient-descent-based multi-model adaptive learning strategy \cite{narendra2000adaptive,narendra2011,narendra2011mmac} is utilized to learn the dynamic parameter estimates as well as the domain of uncertainty in the parameter space. Before a motion is implemented, the library of motion primitives associated with the static nominal estimated parameter is chosen, such that the static nominal parameter estimate is closest to the dynamic parameter estimate at the initial time. Online, the motion is executed using the chosen library of motion primitives equipped with the tubes and it is guaranteed that the overall motion will always lies inside the tube. It is analytically proved that the optimal motion plan associated with the actual uncertain system always lie within the tube around the motion primitive utilized for the executed motion plan. The overall benefits of the proposed strategy are as follows:
\begin{enumerate}
    \item The proposed motion planner generates robust motion plan utilizing estimates of uncertain system parameters and the domain of uncertainty. A multi-model adaptive identification strategy is used to improve the quality of the parameter estimates as well as the uncertainty domain online utilizing measured system data, thus improving the  quality of the implemented motion plan over time.
    \item It is analytically proved that the multi-model adaptive identification strategy leads to reduction in the diameter of the model set containing the uncertain parameter. This in turn indicates that the size of the tubes and hence margin to obstacle are also gradually reduced thus potentially generating better routes over time.
    \item Since the tubes always contain the optimal motion plan associated with the uncertain system, gradual reduction in the size of tubes results in a reduction in deviation of the executed motion plan from the true optimal plan. 
\end{enumerate}
To summarize, the proposed motion planner is capable of handling model uncertainty and operational constraints simultaneously and incorporates robustness along with learning capabilities to ensure safe motion as well as gradual improvement of the planning performance over time. The claimed benefits of this approach are justified theoretically as well as through a suitable simulation experiment.\\~\\
\textbf{Notation:} Given two sets \begin{small}$\mathcal{P},\mathcal{Q}$\end{small} then \begin{small}$\mathcal{P}\oplus\mathcal{Q}\triangleq \big \{a+b:a\in \mathcal{P}, b\in \mathcal{Q}\big \}$, $\mathcal{P}\ominus \mathcal{Q}\triangleq \big \{a: a\oplus \mathcal{Q}\subseteq \mathcal{P} \big \}$\end{small}. $\mathbb{R}$ and $\mathbb{N}$ are the sets of all real and integer numbers, respectively and $\mathbb{N}_{[m:n]}\triangleq \{m,m+1,\cdots,n-1,n\}, \forall m,n\in \mathbb{N}$ and $m<n$. The notations $\|\cdot\|$ and $\|\cdot\|_{\infty}$ denote the $2$-norm and the infinity norm of the argument vectors/matrices, respectively. Given a vector $x\in \mathbb{R}^n$ and a non-singular matrix $Q\in \mathbb{R}^{n\times n}$, the  weighted vector norm $\|x\|_Q$ is defined as $\|x\|_Q\triangleq \sqrt{x^TQx}$. A set $P$ defined as $P = co\{p_1, p_2,\cdots,p_m\}$, where $(p_i, i) \in \mathbb{R}^n \times \mathbb{I}^+_m$ is a convex hull in $\mathbb{R}^n$ with $\{p_1,\cdots,p_m\}$ as the vertices.  A compact ball of radius $r \in \mathbb{R}$ and center $M \in \mathbb{R}^{n\times m}$ is defined as $\mathcal{B}_r(M) = \{M' | \|M' - M\|_F\leqslant r\}$. The diameter of a compact set $P$ is defined as diam$(P)=$ max$\{\|m -n\|_F : m, n \in P\}$.
\vspace{-0.05in}
\section{Problem Formulation}
The paper focuses on the motion-planning problem for uncertain nonlinear systems in the form 
\vspace{-0.04in}
\begin{equation}\label{eq:nonlinear uncertain dynamics}
\dot{x}(t)=\theta^x\phi(x(t))+\theta^uu(t)=\Theta X(t), \quad x(t_0)=x_{ini}
\vspace{-0.0in}
\end{equation}
\noindent where $\phi(x(t))$ is the regressor vector and $x(t)$, $u(t)$ are the system state and the control input vectors, respectively. The model parameters $\theta^x$ and $\theta^u$ are considered to be uncertain and are represented by a lumped parameter matrix $\Theta \triangleq [\theta^x,\theta^u]$. The lumped regressor $X(t)$ is defined as $X(t)\triangleq [\phi(x(t))^T,u(t)^T]^T$. The state and control inputs are subject to the following safety constraints
\vspace{-0.04in}
\begin{equation}
x(t)\in \mathcal{X}\backslash \mathcal{O}\triangleq \overline{\mathcal{X}}\subset \mathbb{R}^n;\quad u(t)\in \mathcal{U}\subset \mathbb{R}^m \nonumber
\vspace{-0.04in}
\end{equation}
\noindent 
where the region $\mathcal{O}\subset \mathcal{X}\subseteq \mathbb{R}^n$ represents the obstacle region in the state space, which must be avoided by the system \eqref{eq:nonlinear uncertain dynamics}. 
It is considered that the parametric uncertainty satisfies the following 
\vspace{-0.0in}
\begin{equation}\label{eq:theta bound}
\Theta \in \Psi\triangleq co\{\psi_2,\psi_2,\cdots,\psi_q\} , \forall i \in \mathbb{I}^+_2
\end{equation}
\noindent where $\psi_i\triangleq [\theta^x_i,\theta^u_i],\forall i\in \mathbb{I}^+_q$ are known. 
\begin{assum}
It is assumed that $\theta^u_i, \forall i\in \mathbb{I}^+_q$ are invertible.
\end{assum}
The considered motion-planning problem can be formulated as a constrained optimal control problem (COCP) as follows
\begin{subequations}\label{eq:path planning COCP}
\begin{align}
\min_{x(t),u(t),T_f} &\ \ J=\int_{0}^{T_f}l(x(t),u(t))dt \nonumber \\
& x(t_0)=x_{ini}; \quad x(T_f)=x_f \\
& \dot{x}(t)=\theta^x \Phi(x(t))+\theta^uu(t)\label{eq:dynamic constraint} \\
& x(t)\in \overline{\mathcal{X}};\quad u(t)\in \mathcal{U}
\end{align}
\end{subequations} 
The running cost $l(x,u)$ is chosen to define the performance measure $J$. The uncertainty $\Theta$ in \eqref{eq:dynamic constraint} is considered to satisfy \eqref{eq:theta bound}. The COCP in \eqref{eq:path planning COCP} is designed to return a feasible trajectory for the system in \eqref{eq:nonlinear uncertain dynamics} to travel from the initial state $x_{ini}$ to a desired final state $x_f$ while respecting all the imposed safety constraints. However, COCP \eqref{eq:path planning COCP} is ill-posed since the uncertainty term $\Theta$ in \eqref{eq:dynamic constraint} is not fully known and thus COCP \eqref{eq:path planning COCP} only illustrates a conceptually defined problem. This motivates the objective of this work, which is to reformulate the COCP \eqref{eq:path planning COCP} to account for the effect of model uncertainty while ensuring constraint satisfaction and to design an adaptive lattice-based motion planner, which will provide a feasible solution for the COCP.

\section{Adaptive Lattice-based Motion Planner}\label{sec-DPRMP}
In this section, nominal dynamics as well as estimated dynamics corresponding to the uncertain plant \eqref{eq:nonlinear uncertain dynamics} are introduced and an adaptive law is proposed which updates the estimated parameters as well as the domain of uncertainty. Subsequently an adaptive lattice-based motion planner is proposed which utilized the adaptive estimated dynamics for computing motion primitives.
\subsection{Nominal dynamics}
The domain of uncertainty $\Psi$ is discretized to obtain the set of discrete parameter estimates denoted as $\Psi_d(N)$ and defined as follows 
\begin{equation}\label{eq:discretized psi}
\Psi_d(N)\triangleq \{\overline{\Theta}_i, i\in \mathbb{I}_N\ |\ \overline{\Theta}_i\in \Psi, \forall i\in \mathbb{I}_N\}
\end{equation}
\noindent where $N$ is a finite integer indicating the number of discrete elements in $\Psi_d(N)$. The nominal dynamics associated with any $\overline{\Theta}_i, i\in \mathbb{I}_N$ is defined as follows
\begin{equation}\label{eq:nominal theta parameterized dynamics}
\dot{\overline{x}}_i(t)=\overline{\theta}^x_i\phi(\overline{x}_i(t))+\overline{\theta}^u_i\overline{u}_i(t); \ {\overline{x}}_i(t_0)=x(t_0)
\end{equation}
\noindent where $\phi(\cdot)$ is the same function used in \eqref{eq:nonlinear uncertain dynamics}. The nominal dynamics is used (later in this section) to compute motion primitives, which would be utilized by the adaptive lattice-based motion planner to compute the overall motion.
\subsection{Estimated dynamics and adaptive model set}
An estimated lumped parameter matrix $\widehat{\Theta}(t))$, corresponding to the uncertain parameter $\Theta$ is considered, where $\widehat{\Theta}(t)\triangleq [\widehat{\theta}^x(t),\widehat{\theta}^u(t)]$. Utilizing $\widehat{\Theta}(t)$, an estimated system is considered as follows
\begin{align}\label{eq:estimated dynamics}
\dot{\widehat{x}}(t)=\widehat{\theta}^x(t)\phi(x(t))+\widehat{\theta}^u(t)\widehat{u}(t)=\widehat{\Theta}(t)\widehat{X}(t)
\end{align}
where $\widehat{X}(t)=[x(t)^T,\widehat{u}(t)^T]^T$ and the estimated parameter $\widehat{\Theta}(t)=[\widehat{\theta}^x(t) \ \widehat{\theta}^u(t)]$ belongs to an adaptive model set $\mathcal{S}(t)$, whose vertices are updated following an adaptive update law (discussed later in this section) such that the following holds for all time
\begin{subequations}\label{eq:model set definition}
\begin{align}
& \widehat{\Theta}(t)\in \mathcal{S}(t)\triangleq \textnormal{co}\{\widehat{\psi}_1(t),\widehat{\psi}_2(t),\cdots,\widehat{\psi}_q(t),\} \label{eq:model set definition a}\\
& \widehat{\psi}_i(t) \triangleq [\widehat{\theta}^x_i(t),\widehat{\theta}^u_i(t)], \forall i\in \mathbb{I}_q^+\\
& \mathcal{S}(t_0)= \Psi \ \Rightarrow \ \widehat{\psi}_i(t_0)=\psi_i, \forall i\in \mathbb{I}_q^+ \label{eq:model set definition b}\\
& \widehat{\Theta}(t)=\sum_{i=1}^q \gamma_i \widehat{\psi}_i(t); \quad \sum_{i=1}^q \gamma_i=1 \label{eq:model set definition c}
\end{align}
\end{subequations}
Let the estimated dynamics associated with the vertices $\widehat{\psi}_i(t)$ be defined as follows
\begin{subequations}\label{eq:estimated dynamics vertices}
\begin{align}
& \dot{\widehat{x}}_i(t)=\widehat{\theta}^x_i(t)\phi(x(t))+\widehat{\theta}^u_i(t){u}(t)-\widehat{\theta}^u(t)\tilde{u}_i(t); \label{eq:estimated dynamics vertices a} \\
& \tilde{u}_i(t)\triangleq u(t)-\widehat{u}_i(t); \label{eq:estimated dynamics vertices b}\\
& \widehat{u}(t)=\sum_{i=1}^q \gamma_i\widehat{u}_i(t); \ \tilde{u}(t)= \sum_{i=1}^q \gamma_i\tilde{u}_i(t)\label{eq:estimated dynamics vertices c}
\end{align}
\end{subequations}
where $\gamma_i, \forall i\in \mathbb{I}^+_q$ satisfies \eqref{eq:model set definition c}. From \eqref{eq:estimated dynamics}, \eqref{eq:model set definition} and \eqref{eq:estimated dynamics vertices} the following is inferred
\begin{equation}
\sum_{i=1}^q \gamma_i \dot{\widehat{x}}_i(t)=\dot{\widehat{x}}(t) \Rightarrow \sum_{i=1}^q \gamma_i \widehat{x}_i(t)={\widehat{x}}(t)
\end{equation}
The vertices $\widehat{\psi}_i(t), \forall i \in \mathbb{I}^+_q$ are updated using the following update law and have the following properties
\begin{subequations}\label{eq:new adaptive update law theta without proj}
\begin{align}
& \dot{\widehat{\psi}}_i(t)=\Gamma \tilde{x}_i(t)X(t)^T ;\ X(t)\triangleq [\phi(x(t))^T,u(t)^T]^T ; \label{eq:new adaptive update law theta a}\\
& \tilde{x}_i(t)\hspace{-0.02in}\triangleq \hspace{-0.02in}x(t)\hspace{-0.02in}-\hspace{-0.02in}\widehat{x}_i(t); \ \sum_{i=1}^q \gamma_i {\tilde{x}}^i(t)\hspace{-0.02in}=\hspace{-0.02in}x(t)\hspace{-0.02in}-\hspace{-0.02in}\widehat{x}(t)\hspace{-0.02in}=\hspace{-0.02in}\tilde{x}(t)\label{eq:new adaptive update law theta b}\\
& \sum_{i=1}^q \hspace{-0.05in}\gamma_i\dot{\widehat{\psi}}_i(t)\hspace{-0.025in}=\hspace{-0.025in}\Gamma \sum_{i=1}^q \hspace{-0.05in}\gamma_i\tilde{x}_i(t)X(t)^T\hspace{-0.025in}=\hspace{-0.025in}\Gamma \tilde{x}(t)X(t)^T\hspace{-0.025in}=\hspace{-0.025in}\dot{\widehat{\Theta}}(t)
\label{eq:new adaptive update law theta c}\end{align}
\end{subequations}
Where $\Gamma > 0$ is the rate of parameter adaptation. The parameter estimation errors $\tilde{\psi}_i(t)$ associated with the estimated parameters $\widehat{\psi}_i(t)$ and the parameter estimation error $\tilde{\Theta}(t)$ associated with the estimated parameter $\widehat{\Theta}(t)$ are defined as follows 
\begin{subequations}
\begin{align}
&\tilde{\psi}_i(t)=\Theta-\widehat{\psi}_i(t)\label{eq:tilde psi i definition} \\
&=[\theta^x\ \theta^u]-[\widehat{\theta}^x_i(t)\ \widehat{\theta}^u_i(t)]=[\tilde{\theta}^x_i(t)\ \tilde{\theta}^u_i(t)], \forall i\in \mathbb{I}^+_q; \\
&\tilde{\Theta}(t)=\sum_{i=1}^q \gamma_i \tilde{\psi}_i(t)= \Theta-\sum_{i=1}^q \gamma_i\widehat{\psi}_i(t)= \Theta-\widehat{\Theta}(t)\label{eq:tilde theta definition}
\end{align}
\end{subequations}
\begin{theorem}\label{thm: Adaptive learning analysis}
If the inputs $\tilde{u}_i(t)$  and hence $\tilde{u}(t)$ is chosen as 
\begin{equation}\label{eq:tilde u equation}
\tilde{u}_i(t)\hspace{-0.02in}=\hspace{-0.02in}-\widehat{\theta}^u(t)^{-1}k\tilde{x}_i(t)\stackrel{\eqref{eq:estimated dynamics vertices c},\eqref{eq:new adaptive update law theta b}}{\Rightarrow} \tilde{u}(t)\hspace{-0.02in}=\hspace{-0.02in}-\widehat{\theta}^u(t)^{-1}k\tilde{x}(t)
\end{equation}
\noindent where $k$ is the controller gain and the estimated parameter $\psi_i(t), \forall i \in \mathbb{I}^+_q$ (and hence $\widehat{\Theta}(t)$) is updated following 
\eqref{eq:new adaptive update law theta}, then the following hold true
\begin{itemize}
\item[1.] The parameter estimation errors $\tilde{\psi}_i(t), \forall i\in \mathbb{I}^+_q$ (hence the parameter estimation error $\tilde{\Theta}(t)$) and the state estimation errors $\tilde{x}_i(t), \forall i\in \mathbb{I}^+_q$ (hence the state estimation error $\tilde{x}(t)$) are bounded.
\item[2.] $\tilde{x}_i(t), \forall i\in \mathbb{I}^+_q$ (hence $\tilde{x}(t)$) asymptotically converges to zero.
\end{itemize}
\end{theorem}
\begin{proof}
Utilizing \eqref{eq:new adaptive update law theta b}, the error dynamics associated with the estimated systems in \eqref{eq:estimated dynamics vertices a} can be represented as follows
\begin{align}
\dot{\tilde{x}}_i(t)&=\tilde{\theta}^x_i(t)\phi(x(t))+\tilde{\theta}^u_i(t)u(t)+\widehat{\theta}^u(t)\tilde{u}_i(t) \nonumber \\
&= \tilde{\psi}_i(t)X(t)+\widehat{\theta}^u(t)\tilde{u}_i(t)\label{eq:xtilde i system}
\end{align}
\noindent With the control input $\tilde{u}_i(t)$ defined in \eqref{eq:tilde u equation}, the closed loop system in \eqref{eq:xtilde i system} becomes
\begin{equation}\label{eq:xtilde i system closed-loop}
\dot{\tilde{x}}_i(t)=\tilde{\psi}_i(t)X(t)-k\tilde{x}_i(t)
\end{equation}
To analyze the stability of the closed-loop system in \eqref{eq:xtilde i system closed-loop}, a Lyapunov function candidate is chosen as follows
\begin{equation}
V_{\tilde{x}}(t)=\frac{1}{2}\tilde{x}_i(t)^T\tilde{x}_i(t)+\frac{1}{2}tr\left (\tilde{\psi}_i(t)^T\Gamma^{-1}\tilde{\psi}_i(t)\right )\nonumber
\end{equation}
\noindent The derivative of $V_{\tilde{x}}(t)$ is now obtained as follows
\begin{align}\label{eq:Ve dot derivative}
&\dot{V}_{\tilde{x}}(t)=\tilde{x}_i(t)^T\dot{\tilde{x}}_i(t)+tr\left (\tilde{\psi}_i(t)^T\Gamma^{-1}\dot{\tilde{\psi}}_i(t)\right ) \nonumber \\
&\stackrel{\eqref{eq:tilde psi i definition},\eqref{eq:xtilde i system closed-loop}}{=} \tilde{x}_i(t)^T\left (\tilde{\psi}_i(t)X(t)-k\tilde{x}_i(t)\right )\hspace{-0.02in}-\hspace{-0.02in}tr\left (\tilde{\psi}_i(t)^T\Gamma^{-1}\dot{\widehat{\psi}}_i(t)\right )\nonumber \\
&\stackrel{\eqref{eq:new adaptive update law theta a}}{=} \hspace{-0.03in}-\hspace{-0.04in}k\|\tilde{x}_i(t)\|^2\hspace{-0.02in}+\hspace{-0.02in}\tilde{x}_i(t)^T\tilde{\psi}_i(t)X(t)\hspace{-0.02in}-\hspace{-0.02in}tr\left (\tilde{\psi}_i(t)^T\tilde{x}_i(t)X(t)^T\right )\nonumber \\
&= -k\|\tilde{x}_i(t)\|^2
\end{align}
The second equality is obtained by taking derivative of $\tilde{\psi}_i(t)$ with equating $\dot{\Theta}=0$ (since $\Theta$ is constant). The fourth equality is obtained from the third one by using the properties of $tr(\cdot)$ on matrix products. Since $\dot{V}_{\tilde{x}}(t)$ is negative semi-definite, it is claimed that Claim $1$ of the theorem holds. To prove Claim $2$, Barbalat's lemma (\textit{Lemma} $8.2$, \cite{khalil}) is invoked. Let a function $f(t)$ be defined as $f(t)=\|\tilde{x}_i(t)\|^2$. Now, $\dot{f}(t)=2\tilde{x}_i(t)^T(\tilde{\psi}_i(t)X(t)-k\tilde{x}_i(t))$, which is bounded since $\tilde{\psi}_i(t)$ and $\tilde{x}_i(t)$ are bounded (from the Lyapunov analysis in \eqref{eq:Ve dot derivative}) and $\phi(x(t))$ is bounded for all $x(t)\in \mathcal{X}$. This proves that $f(t)$ is uniformly continuous. Further, the following is obtained using \eqref{eq:Ve dot derivative}
\begin{equation}
\lim_{t\rightarrow \infty}\int_0^t f(t)dt= -\lim_{t\rightarrow \infty}\int_0^t \frac{dV_{\tilde{x}}(t)}{k}= \lim_{t\rightarrow \infty} \frac{V_{\tilde{x}}(t_0)-V_{\tilde{x}}(t)}{k}
\end{equation}
Since $V_{\tilde{x}}(t_0),k$ are finite and $\lim_{t\rightarrow \infty} V_{\tilde{x}}(t)$ is also finite (since $V_{\tilde{x}}$ is positive definite and $\dot{V}_{\tilde{x}}$ is negative semi-definite), it holds that $\lim_{t\rightarrow \infty}\int_0^t f(t)dt$ exists and is finite. Therefore, by invoking Barbalat's lemma, the following holds
\begin{equation}
\lim_{t\rightarrow \infty} f(t)= 0 \ \Rightarrow \ \lim_{t\rightarrow \infty} \tilde{x}_i(t)=0 \nonumber
\end{equation}
\noindent Since $\tilde{\psi}_i(t), \forall i\in \mathbb{I}^+_q$ are bounded, using \eqref{eq:tilde theta definition} it holds that $\tilde{\Theta}(t)$ is also bounded. Similarly since $\tilde{x}_i(t), \forall i\in \mathbb{I}^+_q$ are bounded and asymptotically converging to zero, using \eqref{eq:new adaptive update law theta b} it is claimed that $\tilde{x}(t)$ is also bounded and asymptotically converging to zero. This concludes the proof.
\end{proof}
The adaptive law in \eqref{eq:new adaptive update law theta without proj} is modified to include a projection operator as follows
\begin{equation}\label{eq:new adaptive update law theta}
\dot{\widehat{\psi}}_i(t)=\textnormal{Proj}_{\Psi}(\Gamma \tilde{x}_i(t)X(t)^T) \Rightarrow \dot{\widehat{\Theta}}(t)=\textnormal{Proj}_{\Psi}(\Gamma \tilde{x}(t)X(t)^T)
\end{equation}
where $\textnormal{Proj}_{\Psi}(\cdot)$ returns the orthogonal projection of the argument on the set $\Psi$ , thus ensuring the parameter estimates $\widehat{\psi}_i(t),\forall i\in \mathbb{I}^+_q$ (hence $\widehat{\Theta}(t)$) to stay inside $\Psi$ for all time. Details about the  $\textnormal{Proj}_{(\cdot)}(\cdot)$ operator can be found in \textit{Section $3.10$} \cite{ioannou2006adaptive}. 
\begin{lem}\label{lem:adaptive laws connection}
[\textit{Theorem $3.10.1$}, \cite{ioannou2006adaptive}] It is guaranteed that the projection-based adaptive update law \eqref{eq:new adaptive update law theta} retains all the properties of the adaptive update law in \eqref{eq:new adaptive update law theta without proj}.
\end{lem}

\begin{lem}\label{lem:theta in S}
[\textit{Theorem 1}, \cite{han2011}] If $\Theta\in \mathcal{S}(t_0)$ and $\widehat{x}_i(t_0)=x(t_0), \forall i\in \mathbb{I}^+_q$ then with the adaptive update law \eqref{eq:new adaptive update law theta a}, $\Theta\in \mathcal{S}(t), \forall t> t_0$. 
\end{lem}
\begin{proof}
If $\Theta\in \mathcal{S}(t_0)$, then there exists $\{\alpha_1,\alpha_2,\cdots, \alpha_q\}\in \mathbb{R}^q$, such that 
\begin{equation}
\sum_{i=1}^q \alpha_i=1; \quad \sum_{i=1}^q \alpha_i \widehat{\psi}_i(t_0)=\Theta \nonumber
\end{equation}
\noindent Let the initial choice of $\widehat{\Theta}(t_0)$ be
\begin{equation}\label{eq:thathatnot equal to theta in proof}
\widehat{\Theta}(t_0)=\sum_{i=1}^q \alpha_i \widehat{\psi}_i(t_0)=\Theta 
\end{equation}
If $\widehat{x}_i(t_0)=x(t_0), \forall i\in \mathbb{I}^+_q$ and $\widehat{\psi}_i(t),\forall i\in \mathbb{I}^+_q$ are updated following \eqref{eq:new adaptive update law theta a} (equivalently $\widehat{\Theta}(t)$ is update following \eqref{eq:new adaptive update law theta c}), then the following holds true
\begin{equation}\label{eq:thathatt equal to theta in proof}
\widehat{\Theta}(t_0)=\Theta=\widehat{\Theta}(t)=\sum_{i=1}^q \alpha_i \theta_i(t), \forall t>t_0\ \Rightarrow \ \Theta \in \mathcal{S}(t)
\end{equation}
\noindent This concludes the proof.
\end{proof}
\begin{corollary}\label{cor:xtilde bound}
With the initial condition $\tilde{x}(t_0)=0$, the state estimation error $\tilde{x}(t)$ is bounded as follows
\begin{equation}
\|\tilde{x}(t)\| \leqslant  \sqrt{|\Gamma^{-1}|}\textnormal{diam}(\mathcal{S}(t_0))
\end{equation}
where $\Gamma>0$ is the rate of parameter adaptation, introduced in \eqref{eq:new adaptive update law theta c}
\end{corollary}
\begin{proof}
From the Lyapunov analysis in the proof of \textit{Theorem} \ref{thm: Adaptive learning analysis}, it is concluded that $V_{\tilde{x}}(t)$ is non-increasing and therefore, the following holds true
\begin{align}
& V_{\tilde{x}}(t)\leqslant V_{\tilde{x}}(t_0) \nonumber 
\end{align}
If the initial condition is $\tilde{x}(t_0)=0$ which in turn implies $\tilde{x}_i(t_0)=0, \forall i\in \mathbb{I}^+_q$ (from \eqref{eq:new adaptive update law theta b}), then the following hold true 
\begin{align}\label{eq:xtilde prebound}
& \|\tilde{x}_i(t)\|^2 <  \|\tilde{x}_i(t)\|^2+|\Gamma^{-1}|\|\tilde{\psi}_i(t)\|^2_F\nonumber \\
&\quad \quad \quad \ \leqslant  \|\tilde{x}_i(t_0)\|^2+|\Gamma^{-1}|\|\tilde{\psi}_i(t_0)\|^2_F=|\Gamma^{-1}|\|\tilde{\psi}_i(t_0)\|^2_F \nonumber \\
\Rightarrow & \|\tilde{x}_i(t)\| \leqslant \sqrt{|\Gamma^{-1}|} \|\tilde{\psi}_i(t_0)\|_F
\end{align} 
From \textit{Lemma} \ref{lem:theta in S}, it can be concluded that 
\begin{equation}\label{eq:psii bounds}
\|\tilde{\psi}_i(t_0)\|_F\leqslant \textnormal{diam}(\mathcal{S}(t_0))
\end{equation}
Therefore, utilizing \eqref{eq:new adaptive update law theta b}, \eqref{eq:xtilde prebound} and \eqref{eq:psii bounds} it is proved that the claimed assertions hold.
\end{proof}
\begin{corollary}\label{cor:x and xhat relation}
It is inferred from \textit{Theorem} \ref{thm: Adaptive learning analysis} and \textit{Corollary} \ref{cor:xtilde bound} that with the control input 
\begin{equation}\label{eq:ut and widehat ut relation}
u(t)=\widehat{u}(t)-\widehat{\theta}^u(t)k\tilde{x}(t)
\end{equation}
The states $x(t)$ and $\widehat{x}(t)$ satisfy the following
\begin{equation}
x(t)\in \widehat{x}(t)\oplus \mathcal{B}_{\delta}(0); \ \widehat{x}(t)\rightarrow x(t)\ \textnormal{as} \ t\rightarrow \infty
\end{equation}
where $\delta \triangleq \sqrt{|\Gamma^{-1}|} \textnormal{diam}(\mathcal{S}(t_0))$.
\end{corollary}

\begin{theorem}\label{thm:size of St}
If the vertices of $\mathcal{S}(t_0)$ is updated following \eqref{eq:new adaptive update law theta a}, then diam$(\mathcal{S}(t))\leqslant$ diam$(\mathcal{S}(t_0)), \forall t> t_0$.
\end{theorem}
\begin{proof}
Let the error terms $\Delta \widehat{x}_{(i,j)}(t)$ and $\Delta \widehat{\psi}_{(i,j)}(t)$ be defined for all $i,j\in \mathbb{I}^+_q, i\neq j$, as follows
\begin{subequations}\label{eq:differences in xi and psii}
\begin{align}
&\Delta \widehat{x}_{(i,j)}(t)\triangleq \widehat{x}_i(t)-\widehat{x}_j(t) \\
&\Delta \widehat{\psi}_{(i,j)}(t)\triangleq \widehat{\psi}_i(t)-\widehat{\psi}_j(t)
\end{align}
\end{subequations}
\noindent Utilizing \eqref{eq:estimated dynamics vertices a} and \eqref{eq:differences in xi and psii} the following is computed
\begin{equation}\label{eq:delta x psi dynamics}
\Delta \dot{\widehat{x}}_{(i,j)}(t)=\Delta \widehat{\psi}_{(i,j)}(t)X(t)-\tilde{\theta}^u(t)(\tilde{u}_i(t)-\tilde{u}_j(t))
\end{equation}
where $X(t)=[\phi(x(t))^T,u(t)^T]^T$. Utilizing \eqref{eq:tilde u equation}, \eqref{eq:delta x psi dynamics} is further modified as follows
\begin{equation}\label{eq:delta V req 1}
\Delta \dot{\widehat{x}}_{(i,j)}(t)=\Delta \widehat{\psi}_{(i,j)}(t)X(t)-k\Delta \widehat{x}_{(i,j)}(t)
\end{equation}
Utilizing \eqref{eq:new adaptive update law theta a} the following is obtained
\begin{equation}\label{eq:delta V req 2}
\Delta \dot{\widehat{\psi}}_{(i,j)}(t)=-\Gamma \Delta \widehat{x}_{(i,j)}(t)X(t)^T
\end{equation}
To analyse the evolution of $\mathcal{S}(t)$ for $t\geqslant t_0$, the following Lyapunov function candidate is considered
\begin{align}
V_{\Delta}(t)=& \frac{1}{2} \Delta \widehat{x}_{(i,j)}(t)^T\Delta \widehat{x}_{(i,j)}(t)\nonumber \\
&+\frac{1}{2}tr\Big (\Delta \widehat{\psi}_{(i,j)}(t)^T\Gamma^{-1}\Delta \widehat{\psi}_{(i,j)}(t)\Big )
\end{align}
The derivative of $V_{\Delta}$ is computed as follows
\begin{align}
\dot{V}_{\Delta}(t)\stackrel{\eqref{eq:delta V req 1},\eqref{eq:delta V req 2}}{=}& \Delta \widehat{x}_{(i,j)}(t)^T(\Delta \widehat{\psi}_{(i,j)}(t)X(t)-k\Delta \widehat{x}_{(i,j)}(t)) \nonumber \\
&-tr(\Delta \widehat{\psi}_{(i,j)}(t)^T\Delta \widehat{x}_{(i,j)}(t)X(t)^T) \nonumber \\
=& -k\Delta \widehat{x}_{(i,j)}(t)^T\Delta \widehat{x}_{(i,j)}(t) \label{eq:Vdelta dot}
\end{align}
Therefore, ${V}_{\Delta}(t)$ is non-increasing (since $\dot{V}_{\Delta}(t)$ is negative semi-definite).\par
Let diam$(\mathcal{S}(t_0))=\Delta \widehat{x}_{(i,j)}(t_0)$ for some $i,j\in \mathbb{I}^+_q, i\neq j$. Let diam$(\mathcal{S}(t))=\Delta \widehat{x}_{(m,n)}(t)$ for any arbitrary time instant $t$ such that $t> t_0$. The indices $m,n$ can take different values other than $(i,j)$ satisfying $m,n\in  \mathbb{I}^+_q, m\neq n$. Since ${V}_{\Delta}(t)$ is non-increasing, the following is concluded
\begin{align}
\textnormal{diam}(\mathcal{S}(t))=&\Delta \widehat{x}_{(m,n)}(t)\leqslant \Delta \widehat{x}_{(m,n)}(t_0)\nonumber \\
 \leqslant &\Delta \widehat{x}_{(i,j)}(t_0)=\textnormal{diam}(\mathcal{S}(t_0))
\end{align}
This concludes the proof.
\end{proof}
\subsection{Relationship between the nominal and the estimated dynamics}
Let the error between the states of the estimated system \eqref{eq:estimated dynamics} and the nominal system \eqref{eq:nominal theta parameterized dynamics} be defined as 
\begin{equation}\label{eq:ebarx definition}
\tilde{\bar{x}}_j(t)\triangleq \widehat{x}(t)-\overline{x}_j(t), j\in \mathbb{I}_N 
\end{equation}
\noindent and the error between the estimated parameter $\widehat{\Theta}(t)$ and nominal parameter $\overline{\Theta}_i, i \in \mathbb{I}_N$ is defined as 
\begin{equation}
\tilde{\bar{\Theta}}_j(t)\triangleq \widehat{\Theta}(t)-\overline{\Theta}_j, j\in \mathbb{I}_N \nonumber
\end{equation}
\noindent The dynamics associated with the error are defined as follows
\begin{align}\label{eq:widehat overline error dynamics}
\dot{\tilde{\bar{x}}}_j(t)=&\dot{\widehat{x}}(t)-\dot{\overline{x}}_j(t) \nonumber \\
=&\widehat{\theta}^x(t)\phi(x(t)) - \overline{\theta}^x_j\phi(\overline{x}_j(t))+\widehat{\theta}^u(t)\widehat{u}(t)- \overline{\theta}^u_j\overline{u}_j(t) \nonumber\\
=&\underbrace{\widehat{\theta}^x(t)\phi(x(t)) - \overline{\theta}^x_j\phi(\overline{x}_j(t))+(\widehat{\theta}^u(t)-\overline{\theta}^u_j)\overline{u}_j(t)}_{h(x(t),\overline{x}_j(t),\overline{u}_j(t),\widehat{\Theta}(t))}\nonumber \\
&+\widehat{\theta}^u(t)\underbrace{(\widehat{u}(t)- \overline{u}_j(t))}_{\tilde{\bar{u}}_j(t)} \nonumber\\
=&h(x(t),\overline{x}_j(t),\overline{u}_j(t),\widehat{\Theta}(t))+\widehat{\theta}^u(t)\tilde{\bar{u}}_j(t)
\end{align}
\noindent The variable $h(x(t),\overline{x}(t),t)$ is measurable and therefore, if the control input $\tilde{\bar{u}}_j(t)$ is chosen as 
\begin{equation}\label{eq:ubartilde}
\tilde{\bar{u}}_j(t)=-\widehat{\theta}^{u}(t)^{-1}(h(x(t),\overline{x}_j(t),\overline{u}_j(t),\widehat{\Theta}(t))+k\tilde{\bar{x}}_j(t))
\end{equation}
\noindent $k$ is the control gain as used in the control input in \eqref{eq:tilde u equation}, then the closed-loop dynamics associated with $\tilde{\bar{x}}_j(t)$ is derived from \eqref{eq:widehat overline error dynamics} as
\begin{equation}\label{eq:xbartilde dynamics}
\dot{\tilde{\bar{x}}}_j(t)=-k\tilde{\bar{x}}_j(t)
\end{equation}
\noindent which is asymptotically stable.
\begin{corollary}\label{cor:xhat around xbar}
It is inferred from \eqref{eq:ubartilde} and \eqref{eq:xbartilde dynamics} that if 
\begin{align}\label{eq:uhat ubar relation}
&\widehat{u}(t)=\overline{u}_j(t)-\widehat{\theta}^{u}(t)^{-1}\Big (h(x(t),\overline{x}_j(t),\overline{u}_j(t),\widehat{\Theta}(t))+k\tilde{\bar{x}}_j(t)\Big )
\end{align}
then $\widehat{x}(t)$ tracks $\overline{x}_j(t)$ asymptotically.
\end{corollary}
The control inputs $\widehat{u}(t)$ is obtained from the convex combination of the control input $\widehat{u}^i(t), i\in \mathbb{I}^+_q$ as defined in \eqref{eq:estimated dynamics vertices c}. Therefore, following \eqref{eq:uhat ubar relation}, the control inputs $\widehat{u}_i(t)$ are designed as follows
\begin{subequations}\label{eq:uibar uhat relation overall} 
\begin{align}
&\widehat{u}_i(t)=\overline{u}_j(t)-\widehat{\theta}^{u}(t)^{-1}(h(x(t),\overline{x}_j(t),\overline{u}_j(t),\widehat{\Theta}_i(t))\nonumber \\
&+k\tilde{\bar{x}}_{(i,j)}(t)), j\in \mathbb{I}_N, \forall i \in \mathbb{I}^+_q \label{eq:uibar uhat relation} \\
& h(x(t),\overline{x}_j(t),\overline{u}_j(t),\widehat{\Theta}_i(t))=\widehat{\theta}_i^x(t)\phi(x(t)) - \overline{\theta}^x_j\phi(\overline{x}_j(t))\nonumber \\
&+(\widehat{\theta}_i^u(t)-\overline{\theta}^u_j)\overline{u}_j(t) \\
& \tilde{\bar{x}}_{(i,j)}(t)=\widehat{x}_i(t)-\overline{x}_j(t)
\end{align} 
\end{subequations}
It can be verified that $\widehat{u}_i(t)$ in \eqref{eq:uibar uhat relation} and $\widehat{u}(t)$ in \eqref{eq:uhat ubar relation} satisfies \eqref{eq:estimated dynamics vertices c}.
\subsection{Relationship between the nominal dynamics and the uncertain system}
Let the error between the state $x(t)$ of the uncertain plant and the nominal state $\overline{x}_j(t), j\in \mathbb{I}_N$ be defined as 
\begin{equation}\label{eq: nominal and uncertain state error}
\tilde{\bar{x}}_{n_j}(t)=x(t)-\overline{x}_j(t)
\end{equation} 
\begin{lem}\label{lem:x and xbar ralation}
If $\widehat{x}(t_0)$ is initialized as $\overline{x}_j(t_0)$ and the control input $u(t)$ is chosen as 
\begin{align}
& u(t)=\overline{u}_j(t)\hspace{-0.02in}\nonumber \\
&-\widehat{\theta}^{u}(t)^{-1}\Big (h(x(t),\overline{x}_j(t),\overline{u}_j(t),\widehat{\Theta}(t))+k(x(t)-\bar{x}_j(t))\Big )\nonumber
\end{align}
Then the following holds true
\begin{align}
\bar{\tilde{x}}_{n_j}(t) \in \mathcal{B}_{\delta}(0); \quad \bar{\tilde{x}}_{n_j}(t)\rightarrow 0\ \textnormal{as} \ t\rightarrow \infty
\end{align}
\end{lem}
\begin{proof}
Using \eqref{eq:new adaptive update law theta b} and \eqref{eq:ebarx definition}, the error $\tilde{\bar{x}}_{n_j}(t)$ can represented as
\begin{equation}\label{eq: relation between es}
\tilde{\bar{x}}_{n_j}(t)=\tilde{x}(t)+\tilde{\bar{x}}_j(t)
\end{equation}
It is concluded from \eqref{eq:xbartilde dynamics} that with the input $\tilde{\bar{u}}_j(t)$ in \eqref{eq:ubartilde}, $\tilde{\bar{x}}_j(t)=0,\forall t\in \mathbb{R}$ if $\tilde{\bar{x}}_j(t_0)=0$, which implies $\widehat{x}(t)=\overline{x}_j(t)$ if $\widehat{x}(t_0)=\overline{x}_j(t_0)$. It is further concluded that with the input $\tilde{u}(t)$ in \eqref{eq:tilde u equation}, $x(t)\in \widehat{x}(t)\oplus \mathcal{B}_{\delta}(0)$. Further utilizing \eqref{eq:tilde u equation} and \eqref{eq:ubartilde} (or \eqref{eq:ut and widehat ut relation} and \eqref{eq:uhat ubar relation}), the relationship between $u(t)$ and $\overline{u}_j(t)$ is established as follows
\begin{align}
& u(t)=\overline{u}_j(t)\hspace{-0.02in}\nonumber \\
&-\hspace{-0.02in}\widehat{\theta}^{u}(t)^{-1}\hspace{-0.02in}\Big (\hspace{-0.03in}h(x(t),\overline{x}_j(t),\overline{u}_j(t),\widehat{\Theta}(t))\hspace{-0.03in}+\hspace{-0.03in}k\tilde{\bar{x}}_j(t)\hspace{-0.03in}+\hspace{-0.03in}k\tilde{x}(t)\hspace{-0.03in}\Big )\nonumber \\
&=  \overline{u}_j(t)\nonumber \\
&\underbrace{-\widehat{\theta}^{u}(t)^{-1}\Big (h(x(t),\overline{x}_j(t),\overline{u}_j(t),\widehat{\Theta}(t))+k(x(t)-\bar{x}_j(t))\Big )}_{\nu(x(t),\overline{x}_j(t),\overline{u}_j(t),\widehat{\Theta}(t))}\nonumber
\end{align}
\begin{align}
&\Rightarrow u(t)=  \overline{u}_j(t)+\nu(x(t),\overline{x}_j(t),\overline{u}_j(t),\widehat{\Theta}(t))\label{eq:ubar and u relation}
\end{align}
Therefore, with $u(t)$ chosen following \eqref{eq:ubar and u relation} and $\widehat{x}(t_0)$ initialized as $\widehat{x}(t_0)=\overline{x}_j(t_0)$, the following holds using \eqref{eq: relation between es}, \textit{Corollary} \ref{cor:x and xhat relation}, and \textit{Corollary} \ref{cor:xhat around xbar}
\begin{equation}
\bar{\tilde{x}}_{n_j}(t) \in \mathcal{B}_{\delta}(0); \quad \bar{\tilde{x}}_{n_j}(t)\rightarrow 0\ \textnormal{as} \ t\rightarrow \infty
\end{equation}
\noindent where $\mathcal{B}_{\delta}(0)$ is defined in \textit{Corollary} \ref{cor:x and xhat relation}. This further implies 
\begin{equation}\label{eq:x and xhat bounded region}
x(t) \in \bar{x}_j(t)\oplus \mathcal{B}_{\delta}(0); \quad x(t) \rightarrow \bar{x}_j(t)\ \textnormal{as} \ t\rightarrow \infty
\end{equation}
This concludes the proof.
\end{proof}
\subsection{Reformulated COCP for motion-planning}\label{sec:RCMP}
The COCP \eqref{eq:path planning COCP} is reformulated to include implementable constraints which guarantee constraint satisfaction for the trajectories associated with the uncertain system \eqref{eq:nonlinear uncertain dynamics}. Let $t_0$ be the current time instant, let the updated adaptive model set obtained after the latest run of the motion planner be $\mathcal{S}(t_0)=$co$\{\widehat{\psi}_1(t_0),\widehat{\psi}_2(t_0),\cdots,\widehat{\psi}_q(t_0)\}$. Let the estimated parameter for the next run of the planner is chosen as $\widehat{\Theta}(t_0)=\sum_{i=0}^q \gamma_i \widehat{\psi}_i(t_0)$ , such that $\sum_{i=0}^q \gamma_i=1$. The nominal parameter $\overline{\Theta}_i, i\in \mathbb{I}_N$ is chosen before the current run of the motion planner, as follows
\begin{subequations}\label{eq:overline theta in S}
\begin{align}
&\overline{\Theta} = \argmin_{\overline{\Theta}_i}\Big (\|\overline{\Theta}_i-\widehat{\Theta}(t_0)\| \Big) \\
 \textnormal{s.t.}\quad  & \overline{\Theta}_i \in \mathcal{S}(t_0)\cap \Psi_d, \forall i\in  \mathbb{I}^+_N 
\end{align} 
\end{subequations}
This guarantees that the nominal parameter to be utilized in the motion planner lies within the current model set $\mathcal{S}(t_0)$. The motion-planning COCP is now based on the nominal system dynamics related to $\overline{\Theta}\triangleq [\overline{\theta}^x\ \overline{\theta}^u]$ computed in \eqref{eq:overline theta in S} and is formulated as follows
\begin{subequations}\label{eq:path planning COCP reformulated dp}
\begin{align}
\min_{\overline{x}(t),\overline{u}(t),T_f}& \ \ H=\int_{0}^{T_f}l(\overline{x}(t),\overline{u}(t))dt \nonumber \\
& \dot{\overline{x}}(t)=\overline{\theta}^x\phi(\overline{x}(t))+\overline{\theta}^u\overline{u}(t)\label{eq:dynamic constraint nominal dp} \\
& \overline{x}(t)\in \overline{\mathcal{X}}_{\mathcal{S}};\ \overline{u}(t)\in \mathcal{U}_{\mathcal{S}}  \label{eq:tightened constraints dp}\\
& \overline{x}(t_0)=x(t_0)=x_{ini}\in  \overline{\mathcal{X}}_{\mathcal{S}}\\
& \overline{x}(T_f)=x_f\in \overline{\mathcal{X}}_{\mathcal{S}}
\end{align}
\end{subequations}
where $\overline{\mathcal{X}}_{\mathcal{S}}$ and ${\mathcal{U}}_{\mathcal{S}}$ denote the tightened state and input constraints, respectively, defined as follows
\begin{subequations}\label{eq:tightened constraints def dp}
\begin{align}
&\overline{\mathcal{X}}_{\mathcal{S}} \triangleq \overline{\mathcal{X}}\ominus \mathcal{B}_{\delta}(t_0); \  \mathcal{U}_{\mathcal{S}} \triangleq \mathcal{U}\ominus \nu(\mathcal{X},\mathcal{X}_{\mathcal{S}},\mathcal{U},\mathcal{S}(t_0));\\
& \mathcal{X}_{\mathcal{S}} \triangleq \mathcal{X}\ominus \mathcal{B}_{\delta}(t_0)
\end{align}
\end{subequations} 
where $\nu(\mathcal{X},\mathcal{X}_{\mathcal{S}},\mathcal{B}_{\delta}(t_0))$ is a set computed as follows
\begin{align}
& \nu(\mathcal{X},\mathcal{X}_{\mathcal{S}},\mathcal{U},\mathcal{S}(t_0))\triangleq \nonumber \\
& \{\nu(x,\overline{x},\overline{u},\widehat{\Theta})\ | \forall (x,\overline{x},\overline{u},\widehat{\Theta}) \in \mathcal{X}\times \mathcal{X}_{\mathcal{S}} \times \mathcal{U}\times \mathcal{S}(t_0)\} \nonumber
\end{align}
\noindent The COCP \eqref{eq:path planning COCP reformulated dp} generates a feasible trajectory for the nominal system \eqref{eq:dynamic constraint nominal dp}, such that it travels from the initial position $x_{ini}$ to the desired final position $x_f$, while satisfying the tightened constraints \eqref{eq:tightened constraints dp}. Since $x(t_0)=\overline{x}(t_0)=x_{ini}$, the following is inferred from \eqref{eq:tightened constraints dp} using \eqref{eq:tightened constraints def dp} and \textit{Lemma} \ref{lem:x and xbar ralation}
\begin{align}
& x(t)\in \overline{x}(t)\oplus \mathcal{B}_{\delta}(t_0)\subset \overline{\mathcal{X}};\quad x(T_f)\in x_f\oplus \mathcal{B}_{\delta}(t_0)\subset \overline{\mathcal{X}} \nonumber\\
& u(t)=\overline{u}(t)+\nu(x,\overline{x},\overline{u},\widehat{\Theta})\in \overline{u}(t)\oplus \nu(\mathcal{X},\mathcal{X}_{\mathcal{S}},\mathcal{U},\mathcal{S}(t_0))\subset \mathcal{U}\nonumber  
\end{align}
\subsection{Lattice-based motion planner with robust constraint satisfaction}
The lattice-based motion-planning strategy converts a motion-planning COCP into a discrete graph-search problem by limiting the controls to a discrete subset of available actions, represented using a set of motion primitives. At the outset, the obstacle-free feasible state space is discretized as per a desired discretization. The discretized state space consists of all reachable states, that form the graph. Subsequently, motion primitives are computed, which are feasible state and control trajectories of the concerned system connecting one reachable state with another in the discretized state space.\par
To develop a lattice-based planner for the motion planning problem in \eqref{eq:path planning COCP reformulated dp}, discretized state space $\mathcal{X}_{d}$ is obtained from the obstacle-free tightened feasible state space $\mathcal{X}_{\mathcal{S}}$, defined as
\begin{equation}
\mathcal{X}_{\mathcal{S}} \triangleq \mathcal{X}\ominus \mathcal{B}_{\delta}(t_0) \nonumber
\end{equation} 
The set of motion primitives $\mathcal{M}_p$, associated with the nominal systems based on all the nominal parameters $\overline{\Theta}_i, \forall i\in \mathbb{I}_N$, is then constructed and a motion primitive $m\in \mathcal{M}$ is defined as follows
\begin{equation}\label{eq:mp def}
m=(\overline{x}_i(t),\overline{u}_i(t),\overline{\Theta}_i)\in \mathcal{X}_{\mathcal{S}}\times \mathcal{U}_{\mathcal{S}}\times \Psi_{d}(N),\quad t\in [0,T]
\end{equation}
The motion primitives associated with each $\overline{\Theta}_i, i\in \mathbb{I}_N$ are computed as follows: let $\overline{x}_{k},\overline{x}_{k+1}\in \mathcal{X}_{d}$ be any two neighbouring points in the discretized state space. Then the motion primitives $m_k$ are computed for each $\overline{x}_{k}\in \mathcal{X}_{d}$ and their corresponding neighbourhood points $\overline{x}_{k+1}\in \mathcal{X}_{d}$ by solving the following COCP
\vspace{-0.0in}
\begin{subequations}\label{eq:motion primitive OCP}
\begin{align}
\min_{\overline{x}_i(t), \overline{u}_i(t),T}& \ \ H=\int_{t_0}^{T}l(\overline{x}_i(t),\overline{u}_i(t))dt \nonumber \\
& \overline{x}_i(t_0)=\overline{x}_{k}; \quad \overline{x}_i(T)=\overline{x}_{k+1}\\
& \dot{\overline{x}}_i(t)=\overline{\theta}^x_i\phi(\overline{x}_i(t))+\overline{\theta}^u_i\overline{u}_i(t)\label{eq:motion primitive nominal dp} \\
& \overline{x}_i(t)\in \mathcal{X}_{\mathcal{S}};\ \overline{u}_i(t)\in \mathcal{U}_{\mathcal{S}}  \label{eq:motion primitive tightened constraints dp}
\end{align}
\end{subequations}
\noindent The solution of the COCP \eqref{eq:motion primitive OCP} constitutes a motion primitive $m_p$ (defined in \eqref{eq:mp def}), in which the motion is governed by the following 
\begin{align}\label{eq:state trasitin dynamics dp}
\overline{x}_{k+1}=& f_m(\overline{x}_k,\overline{\Theta}_i,m_k)\triangleq  \overline{x}_k+\int_0^T\hspace{-0.in}(\overline{\theta}^x_i\phi(\overline{x}_i(t))+\overline{\theta}^u_i\overline{u}_i(t))d{t}
\end{align} 
\noindent The motion-planning COCP \eqref{eq:path planning COCP reformulated dp} can now be approximated by the following graph-search problem, posed as a discrete COCP, which is solved online to guarantee a collision free motion
\begin{subequations}\label{eq:DPLP}
\begin{align}
\min_{\{m_k\}_{k=0}^{M-1},M}& \ \ H_m=\sum_{k=0}^{M-1}l_m(\overline{x}_k,m_k) \nonumber \\
& \overline{x}_0=x_{ini}; \quad \overline{x}_M=x_f\in \overline{\mathcal{X}}_{\mathcal{S}} \\
& \overline{x}_{k+1}=f_m(\overline{x}_k,\overline{\Theta},m_k)\label{eq:LP state transition dp} \\
& m_k\in \mathcal{M}(\overline{\Theta}) \label{eq:LP motion primitive dp}\\
& c(\overline{x}_k,m_k)\in \overline{\mathcal{X}}_{\mathcal{S}} \label{eq:LP feasibility constraint dp}
\end{align}
\end{subequations}
where $\mathcal{M}(\overline{\Theta})$ contains the motion primitives associated with $\overline{\Theta}$, which is selected utilizing \eqref{eq:overline theta in S}.\par 
The following proposition proves that the uncertain system states always remain within a region $\mathcal{B}_{\delta}(t_0)$ around the planned motion, obtained from \eqref{eq:DPLP} by combining a finite number of pre-computed motion primitives.
\begin{propos}
Let $x_k=x(kT)$. If $x_0=x(t_0)=x_{ini}$, then the state $x(t)$ of the uncertain system  \eqref{eq:nonlinear uncertain dynamics} satisfies the following for all $k\in \mathbb{I}_{[1:M]}$
\begin{align}
& x_k\in \overline{x}_k\oplus \mathcal{B}_{\delta}(t_0)\nonumber \\
& x({t})=\overline{x}({t})\oplus \mathcal{B}_{\delta}(t_0), \forall {t} \in [(kT,(k+1)T]\nonumber
\end{align}
with the control input $u({t})=\overline{u}({t})+\nu(x(t),\overline{x}(t),e_x(t)), \forall {t} \in [kT,(k+1)T]$ , where the utilized nominal control $\overline{u}({t})$ is encoded in $m_k,\forall k\in \mathbb{I}_{[0:M-1]}$.
\end{propos}
\begin{proof}
The following proof is done by the method of induction. Let 
\begin{equation}\label{eq:induction eq 1 dp}
x_k\in \overline{x}_k\oplus \mathcal{B}_{\delta}(t_0) 
\end{equation}
Then if the constraint \eqref{eq:LP feasibility constraint dp} is satisfied then the following holds for all ${t} \in (k,(k+1)T)$
\begin{align}\label{eq:induction eq 11 dp}
&\overline{x}({t})\in \overline{\mathcal{X}}_{_\mathcal{S}}= \Big( \mathcal{X}/\mathcal{O}\Big )\ominus \mathcal{B}_{\delta}(t_0) \subset \mathcal{X} \ominus \mathcal{B}_{\delta}(t_0) \nonumber \\
\Rightarrow & \overline{x}({t})\oplus \mathcal{B}_{\delta}(t_0)\in \mathcal{X}
\end{align}
\noindent Therefore, it is concluded using \eqref{eq:induction eq 1 dp}, \eqref{eq:induction eq 11 dp} and \textit{Lemma} \ref{lem:x and xbar ralation} that with the control $u({t})=\overline{u}({t})+\nu(x(t),\overline{x}_j(t),\overline{u}_j(t),\widehat{\Theta}(t))$, where $\overline{u}(t)$ is the control action utilized in the active motion primitive $m_k$, the following holds true 
\begin{subequations}\label{eq:induction eq 2 dp}
\begin{align}
& x({t})=\overline{x}({t})\oplus \mathcal{B}_{\delta}(t_0), \forall {t} \in (k,(k+1)T)\\
&x_{k+1}=x(t+T)\in \overline{x}(t+T)\oplus \mathcal{B}_{\delta}(t_0)=\overline{x}_{k+1}\oplus \mathcal{B}_{\delta}(t_0)
\end{align}
\end{subequations}
Since, the initial condition of the uncertain system's state satisfy  $x_0=\overline{x}_0=x_{ini}$ and the origin is an interior point of $\mathcal{B}_{\delta}(t_0)$, by recursively utilizing \eqref{eq:induction eq 1 dp}-\eqref{eq:induction eq 2 dp}, it is proved that the claimed assertions hold.  
\end{proof}
\vspace{-0.1in}
\begin{algorithm}
\vspace{0in}
\caption {\small Adaptive lattice-based motion planner}
\begin{algorithmic}[]
\renewcommand{\algorithmicrequire}{\textbf{Offline:}}
\REQUIRE 
\end{algorithmic}
\vspace{-0.07in}
\begin{itemize}
\item Initialize $\mathcal{S}(t_0)=\Psi$ if $t_0=0$ else $\mathcal{S}(t_0)=\mathcal{S}(t)$.   
\item Compute $\mathcal{B}_{\delta}(t_0)$.
\item Obtain the tightened spaces $\mathcal{X}_{\mathcal{S}}$ and  $\mathcal{U}_{\mathcal{S}}$.
\item Discretize $\mathcal{X}_{\mathcal{S}}$ to obtain $\mathcal{X}_{d}$.
\item Obtain $\overline{\Theta}$ following \eqref{eq:overline theta in S}.
\item Obtain $\mathcal{M}(\overline{\Theta})$ from $\mathcal{M}$.
\end{itemize}
\vspace{0in}
\begin{algorithmic}[1]
\renewcommand{\algorithmicensure}{\textbf{Online:}}
\ENSURE \hspace{-0.02in}
\STATE Take inputs $x_{ini}$, $x_f$ and $\mathcal{O}$.
\STATE Solve COCP \eqref{eq:DPLP}.
\STATE Apply the control $u(t)=\overline{u}_j(t)+\nu(x(t),\overline{x}_j(t),\overline{u}_j(t),$ $\widehat{\Theta}(t))$ to \eqref{eq:nonlinear uncertain dynamics} and $\widehat{u}_i(t)=\overline{u}_j(t)-\widehat{\theta}^{u}(t)^{-1}(h(x(t),\overline{x}_j(t),$ $\overline{u}_j(t),\widehat{\Theta}_i(t))$ to \eqref{eq:uibar uhat relation overall} .
\STATE Measure $x(t)$ and compute $\widehat{x}(t)$.
\STATE Update $\mathcal{S}(t)$ following \eqref{eq:new adaptive update law theta}.
\end{algorithmic}
\end{algorithm}
\vspace{-0.15in}
\subsection{Discussion on optimality}
In this section, the optimality of the solution obtained by solving the COCP \eqref{eq:motion primitive OCP} is analysed, with respect to the optimal solution for the uncertain system \eqref{eq:nonlinear uncertain dynamics}, which is considered to exist and would be a known signal if $\Theta(\theta)$ would be known. 
Considering the same discretized state space $\mathcal{X}_{d}$ as considered for COCP \eqref{eq:motion primitive OCP} and $\Theta(\theta)$ to be known, the optimal motion primitive for the system in \eqref{eq:nonlinear uncertain dynamics} would be computed by solving the following
\begin{subequations}\label{eq:motion primitive optimal  OCP}
\begin{align}
\min_{x(t), {u}(t),T_f}& \ \ H=\int_{t_0}^{T_f}l(\overline{x}(t),\overline{u}(t))dt \nonumber \\
& {x}(t_0)={x}_{ini}; \quad {x}(T_f)={x}_f\\
& \dot{{x}}(t)=\theta^x\phi(x(t))+\theta^uu(t)\label{eq:motion primitive optimal nominal dp} \\
& {x}(t)\in \mathcal{X};\ {u}(t)\in \mathcal{U}  \label{eq:motion primitive optimal tightened constraints dp}
\end{align}
optimal  OCP.
\end{subequations}
where ${x}_k$ is any point in the discretized state space and ${x}_{k+1}\in \mathcal{X}_{d}$ is any point in the neighbourhood of ${x}_k$. The following theorem establishes the relation between the optimal solutions from  \eqref{eq:motion primitive OCP} and \eqref{eq:motion primitive optimal  OCP}.
\begin{theorem}\label{thm:optimality}
For $x_k=\overline{x}_k$, if $\Theta, \overline{\Theta}\in \mathcal{S}(t_0)$, then the solution of COCP \eqref{eq:motion primitive optimal  OCP} will lie within the tube dictated by $\mathcal{B}_{\delta}(t_0) $, around the solution of the COCP \eqref{eq:motion primitive OCP}. 
\end{theorem}
\begin{proof}
Let $x^*(t)$ be the optimal state trajectory and $u^*(t)$ be the optimal control input computed by solving the COCP \eqref{eq:motion primitive optimal  OCP}, and the optimal dynamics are represented as
\begin{equation}
\dot{x}^*(t)=\theta^x\phi(x^*(t))+\theta^uu^*(t)\nonumber
\end{equation} 
The dynamics for any trajectory obtained by solving \eqref{eq:DPLP} with respect to a nominal parameter $\overline{\Theta}=[\overline{\theta}^x\ \overline{\theta}^u]\in \mathcal{S}(t_0)$ is given as 
\begin{equation}\label{eq:optimality test nominal dynamics}
\dot{\overline{x}}(t)=\overline{\theta}^x\phi(\overline{x}(t))+\overline{\theta}^u\overline{u}(t)
\end{equation}
If $\Theta, \overline{\Theta}\in \mathcal{S}(t_0)$ and $\tilde{\bar{x}}_n(t_0)=x(t_0)-\overline{x}(t_0)\in \mathcal{B}_{\delta}(t_0) $, then utilizing \textit{Lemma} \ref{lem:x and xbar ralation} it holds that a feedback control policy $\nu(x^*(t),\overline{x}_j(t),\overline{u}_j(t),\widehat{\Theta}(t))$ can be computed. It is guaranteed that with a feasible solution \eqref{eq:DPLP} computed as
\begin{subequations}\label{eq:optimality test}
\begin{align}
&\overline{u}(t)=u^*(t)-\nu(x^*(t),\overline{x}_j(t),\overline{u}_j(t),\widehat{\Theta}(t))\\
\vspace{0.1in}
&\hspace{-0.1in}\textnormal{the following holds}\nonumber \\
\vspace{0.1in}
&x^*(t)\in \overline{x}(t)\oplus \mathcal{B}_{\delta}(t_0), \quad \textnormal{if}\quad x(t_0)\in \overline{x}(t_0)\oplus \mathcal{B}_{\delta}(t_0) 
\end{align}
\end{subequations}
Let $\mathcal{V}$ be a set of feasible solutions of \eqref{eq:DPLP} defined as
\begin{equation}
\mathcal{V}\triangleq \{(\overline{x},\overline{u})\ : \ \eqref{eq:optimality test}\}
\end{equation}
Let $(\overline{x}^*(t),\overline{u}^*(t))$ represent the optimal solution obtained by solving COCP \eqref{eq:DPLP}. Since $\Theta, \overline{\Theta}\in \mathcal{S}(t_0)$, $x^*(t_0)=x_k=\overline{x}_k=\overline{x}^*(t_0)$, $u^*(t)\in \mathcal{U}$ and $\overline{u}^*(t)\in \mathcal{U}_{\mathcal{S}}$, then utilizing \textit{Lemma} \ref{lem:x and xbar ralation} it holds that there exists a controller $\nu(x^*(t),\overline{x}_j(t),\overline{u}_j(t),\widehat{\Theta}(t))=u^*(t)-\overline{u}^*(t)$ such that 
\begin{equation}
(\overline{x}^*(t),\overline{u}^*(t)) \in \mathcal{V}
\end{equation}
This proves that the claimed assertion holds.
\end{proof}
\begin{figure*}[t]
    \centering
    \begin{subfigure}{0.32\textwidth}
        \includegraphics[width=\linewidth]{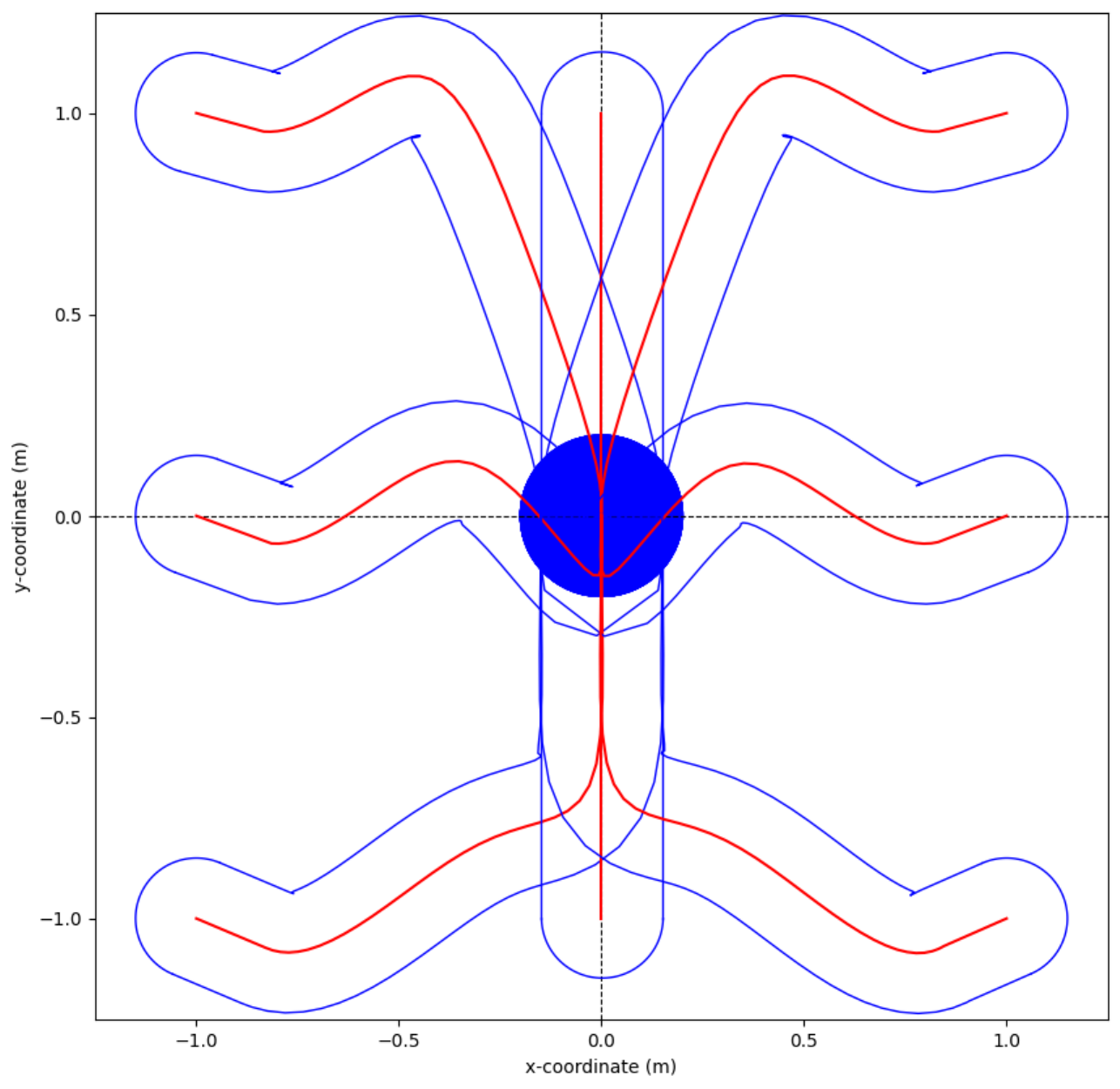}
        \caption{Motion Primitives at initial time}
        \label{fig:MP1}
    \end{subfigure}
    \hfill
    \begin{subfigure}{0.32\textwidth}
        \includegraphics[width=\linewidth]{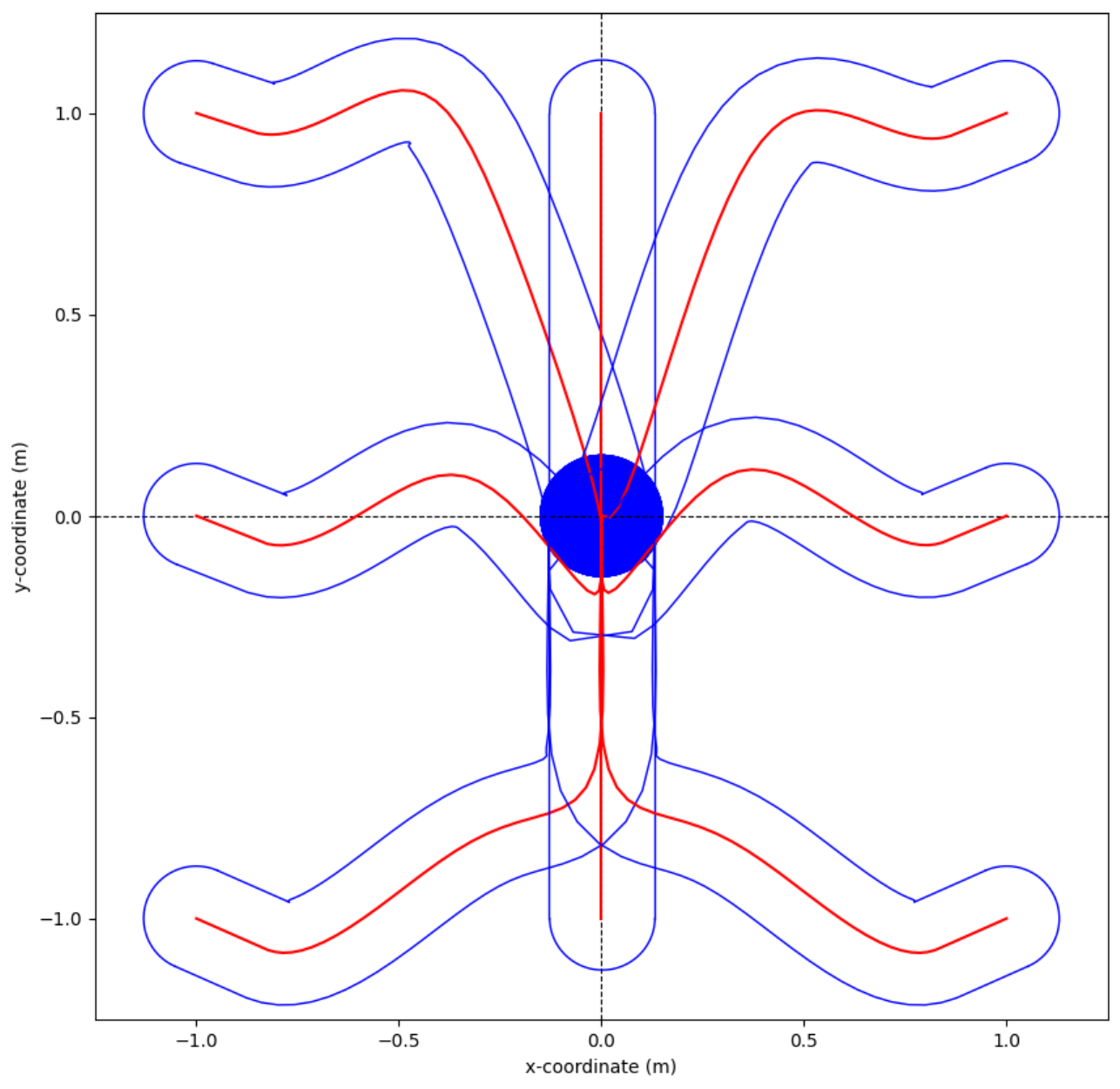}
        \caption{Motion Primitives after 3 executions}
        \label{fig:MP2}
    \end{subfigure}
    \hfill
    \begin{subfigure}{0.32\textwidth}
        \includegraphics[width=\linewidth]{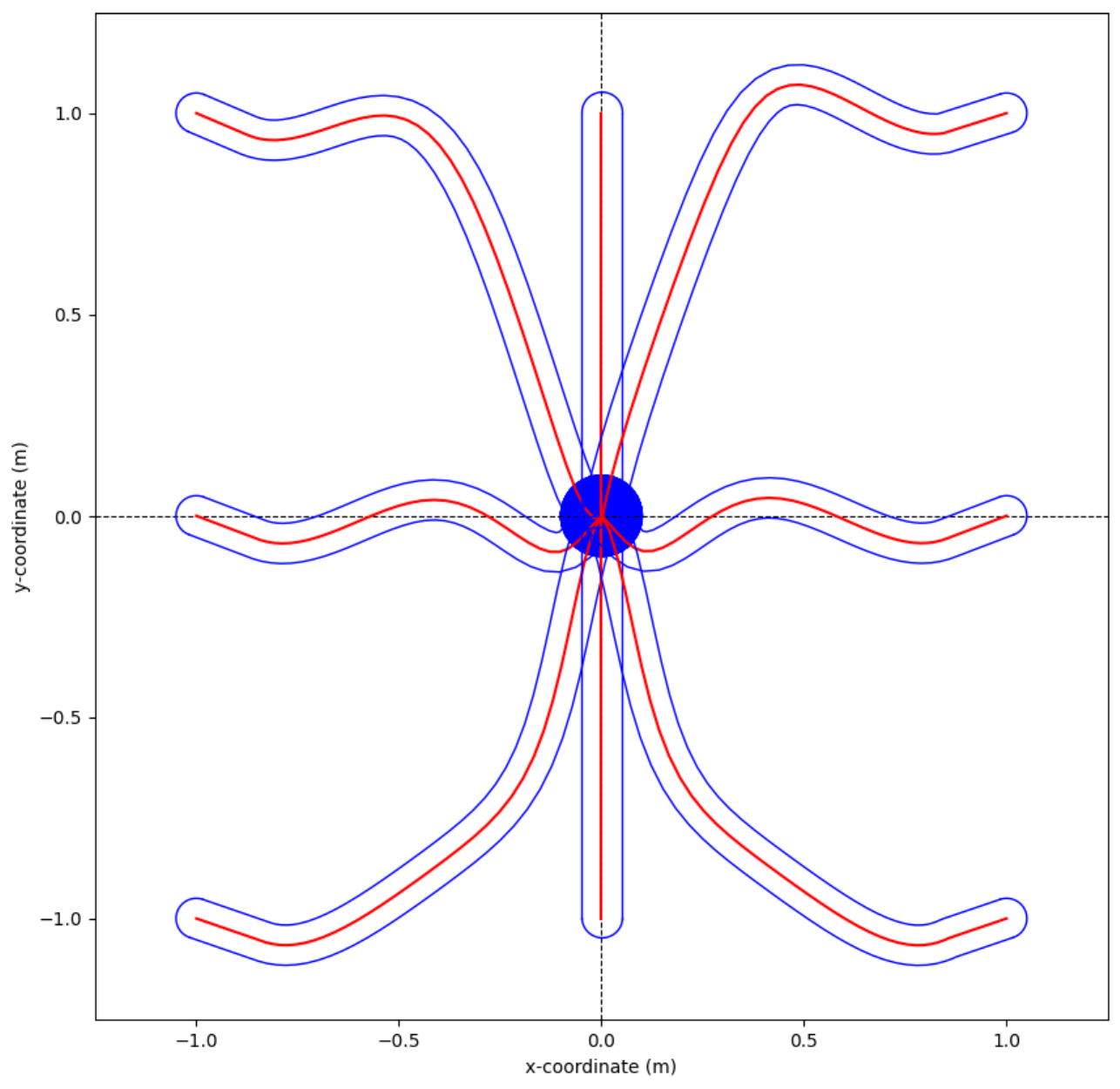}
        \caption{Motion Primitives after 9 executions}
        \label{fig:MP3}
    \end{subfigure}
    \caption{Comparison of motion primitives at different stages}
    \label{fig:MP_all}
\end{figure*}
\begin{corollary}\label{cor:tube reduction and optimality}
Let $(\overline{x}^*(t),\overline{u}^*(t))$ and $({x}^*(t),{u}^*(t))$ represent the optimal solutions of the COCPs \eqref{eq:DPLP} and \eqref{eq:motion primitive optimal  OCP}, respectively. Then, if $\mathcal{S}(t)$ is updated following \eqref{eq:new adaptive update law theta}, then $(\overline{x}^*(t),\overline{u}^*(t))\rightarrow ({x}^*(t),{u}^*(t))$.
\end{corollary}
\begin{proof}
It is inferred from \textit{Theorem} \ref{thm:optimality} that the optimal solutions of COCPs \eqref{eq:DPLP} and \eqref{eq:motion primitive optimal  OCP} stay within a region, dictated by $\mathcal{B}_{\delta}(t_0) $. It is claimed in \textit{Theorem} \ref{thm:size of St}, that the size of $\mathcal{S}(t) $ reduces over time if $\mathcal{S}(t)$ is updated utilizing \eqref{eq:new adaptive update law theta}. This in turn leads to reduction in the size of $\mathcal{B}_{\delta}(t_0)$. Therefore, utilizing \textit{Theorem} \ref{thm:optimality}, it is inferred that the claimed assertion holds.
\end{proof}
\begin{remark}
The overall optimality of the solution obtained from the motion planning COCP \eqref{eq:DPLP} is dependent on the resolution of the discretized state space $\mathcal{X}_{d}$. 
\end{remark}
\begin{remark}
From \textit{Corollary} \ref{cor:tube reduction and optimality} it is inferred that updating $\mathcal{S}(t)$ following \eqref{eq:new adaptive update law theta} leads to possible reduction in the size of tubes over time, which allows the motion planner in \eqref{eq:DPLP} to explore more feasible state space. This in turn potentially allows the planner to exploit the navigable regions which would have been infeasible with a large tube size, thus potentially improving the quality of the overall motion plan.  
\end{remark}
\section{Simulation Results}
To verify the effectiveness of the proposed control scheme the following quadrotor dynamics is considered\cite{dronedyna}:
\begin{align*}
\ddot{\phi} &= a_1 \dot{\theta} \dot{\psi} + c_1 \dot{\theta} \Omega +  b_1 u_1 \\
\ddot{\theta} &= a_2 \dot{\phi} \dot{\psi} + c_2 \dot{\phi} \Omega +  b_2 u_2 \\
\ddot{\psi} &= a_3 \dot{\phi} \dot{\theta} +  b_3 u_3 \\
\ddot{x} &= \frac{u_4 u_x}{m} \\
\ddot{y} &= \frac{u_4 u_y}{m} \\
\ddot{z} &= \frac{u_4 \cos\phi \cos\theta - mg}{m}
\end{align*}
where the symbols are defined in Table \ref{tab:sim_params} with
\begin{align*}
u_x &= \sin\psi \sin\phi + \cos\psi \sin\theta \cos\phi \\
u_y &= \sin\psi \sin\theta \cos\phi - \cos\psi \sin\phi
\end{align*}
\vspace{-0.05in}
\begin{table}[h!]
\centering
\caption{Nomenclature of Quadrotor System Parameters}
\begin{tabular}{p{4cm} p{3.8cm}}
\toprule
\textbf{Parameter} & \textbf{Value} \\
\midrule
$(x,~ y,~ z)$ & Quadrotor position in Earth coordinate frame \vspace{0.03in}\\
$(\phi, ~ \theta, ~ \psi)$ & Quadrotor roll, pitch and yaw angles\vspace{0.03in}\\
Mass of the quadrotor ($m$) & 0.61 kg \vspace{0.02in}\\
Moment of inertia about x-axis ($I_{xx}$) & $1.54 \times 10^{-2}$ kg·m$^2$ \vspace{0.03in}\\
Moment of inertia about y-axis ($I_{yy}$) & $1.54 \times 10^{-2}$ kg·m$^2$ \vspace{0.03in}\\
Moment of inertia about z-axis ($I_{zz}$) & $3.09 \times 10^{-2}$ kg·m$^2$ \vspace{0.03in}\\
Length of the quadrotor arm ($l$) & 0.305 m \vspace{0.03in}\\
 $a_1 = \frac{I_{yy} - I_{zz}}{I_{xx}}$ & -1.0065 \vspace{0.03in}\\
 $a_2 = \frac{I_{zz} - I_{xx}}{I_{yy}}$ & 1.0065 \vspace{0.03in}\\
 $a_3 = \frac{I_{xx} - I_{yy}}{I_{zz}}$ & 0 \vspace{0.03in}\\
 $b_1 = \frac{l}{I_{xx}}$ & 19.8052 \vspace{0.03in}\\
 $b_2 = \frac{l}{I_{yy}}$ & 19.8052 \vspace{0.03in}\\
$b_3 = \frac{1}{I_{zz}}$ & 32.3625 \vspace{0.03in}\\
$c_1 $ & Unknown \vspace{0.03in}\\
$c_2 $ & Unknown \vspace{0.03in}\\
Gravitational coefficient $g$ & $9.81$ m/s$^2$\vspace{0.03in}\\
$(u_1, ~u_2,~u_3)$ & Control input for roll, pitch and yaw \vspace{0.03in}\\
$T$ & Net thrust \vspace{0.03in}\\
$\Omega$ & Net rotor spin rate \vspace{0.0in}\\
\bottomrule
\end{tabular}
\label{tab:sim_params}
\end{table}
\begin{figure}[!h]
    \centering
    \includegraphics[width=0.4\textwidth]{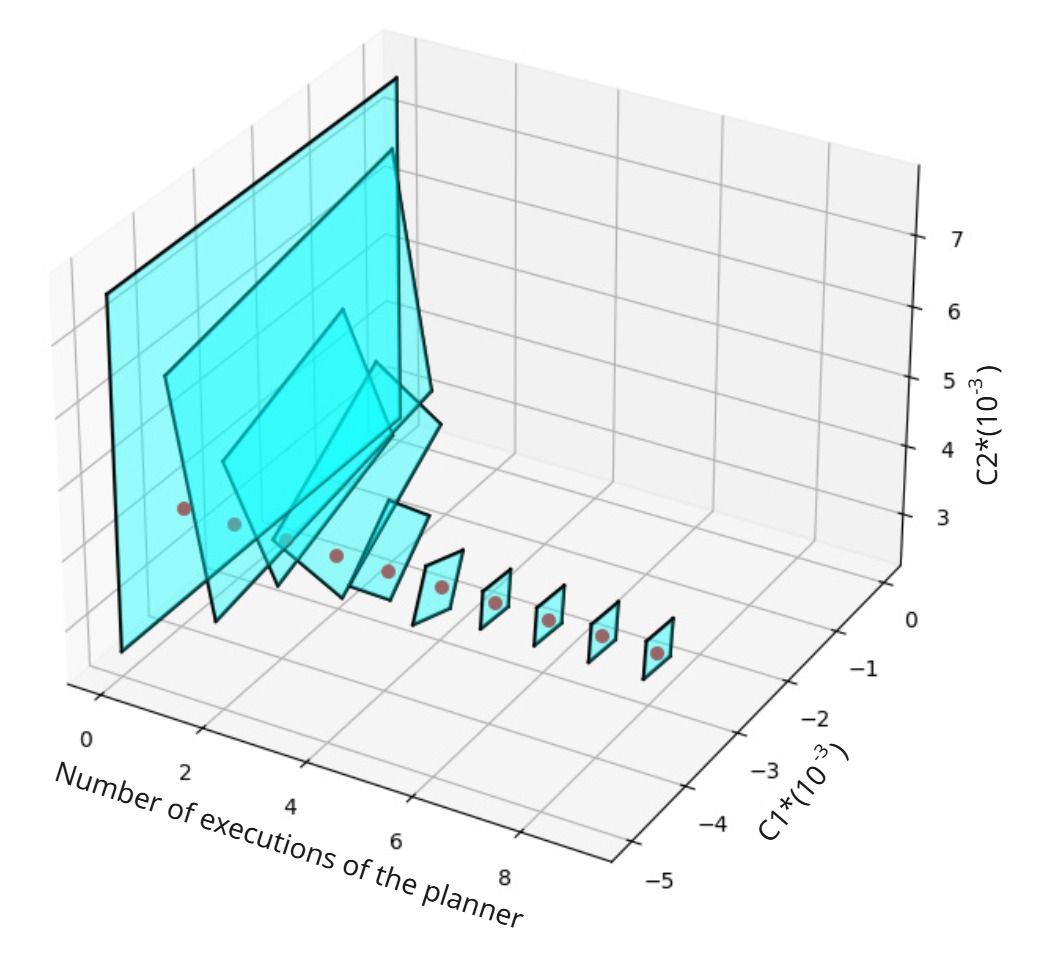}
    \caption{Evolution of bounds (time explicit)}
    \label{fig:MP1}
\end{figure}
\begin{figure}[!h]
    \centering
    \includegraphics[width=0.4\textwidth]{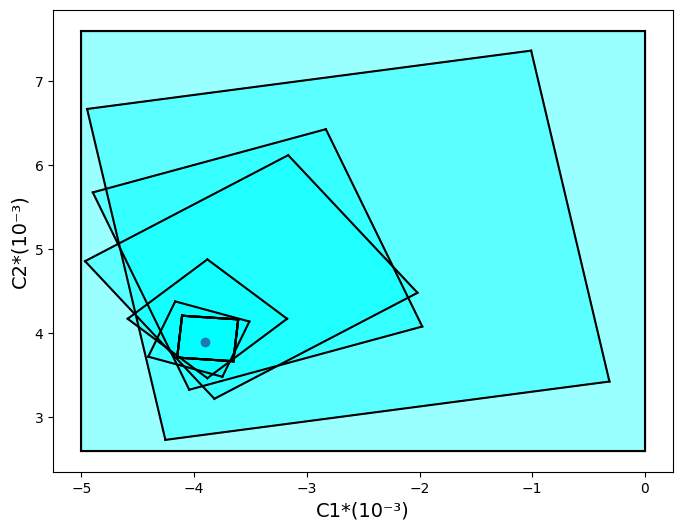}
    \caption{Evolution of bounds (time implicit)}
    \label{fig:MP1}
\end{figure}
\noindent  Initially, multiple estimates of the static nominal model are considered, along with a single estimate for the initial condition of the dynamic model to be estimated. In this experiment the drone parameters $c_i, i\in [1,2]$ are considered to be uncertain\footnote{Ideally all the drone parameters \textit{i.e.} $a_i,b_i, \forall i\in [1,2,3]$ can be chosen to be uncertain along with $c_i,\forall i\in [1,2]$. However, for limitation of plotting beyond 3 dimensional plane, the parametric uncertainty is restricted to be a two dimensional vector.} and the lumped uncertainty parameter $\Theta =[c_1,c_2]^T\in \Psi$, where $\Psi$ is the domain of uncertainty given as 
\begin{align}
    \Psi \triangleq & \textnormal{co}\left \{(-5e^{-3},2.25e^{-3}), (0,2.25e^{-3}), \right .\nonumber \\
    & \left . (0,7.75e^{-3}),(-5e^{-3},7.75e^{-3})\right \} \nonumber
\end{align}
\begin{figure}[!h]
    \centering
    \includegraphics[width=0.4\textwidth]{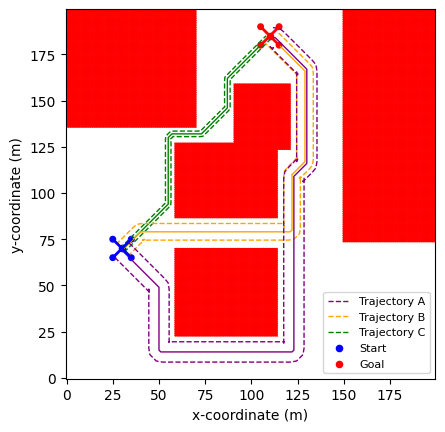}
    \caption{Simulation}
    \label{fig:simulation}
\end{figure}

\noindent The adaptive model set $\mathcal{S}(t)$ is also initialized to be $\mathcal{S}(t_0)=\Psi$. The drone is simulated to repeatedly fly between a predefined start and goal position. Offline, a finite set of nominal parameter values is sampled from the initial model set $\Psi$, and a corresponding library of motion primitives is generated for each nominal parameter. Online, one nominal parameter is selected for each simulation run (following \eqref{eq:overline theta in S}) from the sampled set, and the associated motion primitives are employed to execute the trajectory. The motion primitives are equipped with tubes whose diameter is proportional to the diameter of $\mathcal{S}(t)$. During runtime, an adaptive learning module collects state and input data, and updates the model set as well as the dynamic parameter estimate using a gradient descent-based algorithm as given in \eqref{eq:new adaptive update law theta}. Fig. $2$ and $3$ show the evolution of the model set containing the uncertain parameters at the beginning of executing a motion plan, for multiple subsequent runs of the motion planner. It is seen that the size of the model set gradually reducers over time while true parameter remains within the updated model set fo all time instants. Prior to each subsequent run, the tube size is recomputed based on the reduced diameter of the updated model set, and a new nominal parameter and its corresponding motion primitives are selected using \eqref{eq:overline theta in S}. The plots of the motion primitives along with the corresponding tubes, for various execution instances are shown in Fig. 1. This process is iterated until the model set diameter converges to a stable value significantly smaller than the initial uncertainty bound, ensuring improved model confidence over repeated executions. The overall motion planning performance is shown in Fig. 4. Initially, the bounds of the system parameters are large hence the drone has to take the longest path (Trajectory A). It is seen that as better knowledge of the uncertainty is obtained over time, the conservatism in tube size is also reduced. This enables the motion planner to generate shorter and better path around the obstacles to reach the goal state, as seen after 3 executions (Trajectory B) and 9 executions (Trajectory C) of the motion planner.
\section{Conclusion}
\noindent This paper presents a novel adaptive lattice-based motion planning framework for nonlinear systems with parametric uncertainties operating in cluttered environments. By systematically integrating a multi-model adaptive identification strategy with a structured lattice-based planner, the proposed method enables safe, robust, and resolution-optimal trajectory generation. Unlike conventional robust planners that rely solely on worst-case uncertainty bounds, this approach dynamically refines both the estimated model parameters and the associated uncertainty set using real-time input-output data. The refinement directly improves the quality of the selected motion primitives and allows for progressively tighter safety tubes around the trajectories, ensuring collision avoidance even during the learning phase. Theoretical analysis guarantees that the adaptive learning mechanism reduces the model estimation error and the diameter of the uncertainty set over time. Consequently, the safety margins shrink, and the executed trajectories converge closer to the true optimal path of the system. The effectiveness and adaptability of the proposed planner are demonstrated through simulations involving a drone modeled by Euler-Lagrange dynamics with uncertain parameters, operating in a cluttered workspace. This work contributes a principled hybrid motion planning framework that effectively bridges the gap between learning-based adaptability and robust, model-based safety guarantees. Future work may explore the extension of this approach to more general classes of uncertainties, high-dimensional systems, and real-world hardware implementations.
\section{Acknowledgment}
\noindent We would like to thank Ananth Rachakonda for his valuable assistance in setting up the simulation experiments during the initial phase of this work.

\bibliographystyle{ieeetr}

\begin{thebibliography}{9}
\bibitem{paden2016}
Brian Paden, Michal Čáp, Sze Zheng Yong, Dmitry Yershov, and Emilio Frazzoli,
\textit{A Survey of Motion Planning and Control Techniques for Self-Driving Urban Vehicles},
IEEE Transactions on Intelligent Vehicles, vol. 1, no. 1, pp. 33--55, 2016.

\bibitem{billard2019}
Aude Billard and Danica Kragic,
``Trends and challenges in robot manipulation,''
\textit{Science}, vol. 364, no. 6446, p. eaat8414, 2019.

\bibitem{yurtsever2020}
Ekim Yurtsever, Jacob Lambert, Alexander Carballo, and Kazuya Takeda,
``A survey of autonomous driving: Common practices and emerging technologies,''
\textit{IEEE Access}, vol. 8, pp. 58443--58469, 2020.

\bibitem{lu2018}
Yuncheng Lu, Zhucun Xue, Gui-Song Xia, and Liangpei Zhang,
``A survey on vision-based UAV navigation,''
\textit{Geo-spatial Information Science}, vol. 21, no. 1, pp. 21--32, 2018.

\bibitem{siegwart2011}
Roland Siegwart, Illah Reza Nourbakhsh, and Davide Scaramuzza,
\textit{Introduction to Autonomous Mobile Robots},
MIT Press, 2011.

\bibitem{karaman2011}
Sertac Karaman and Emilio Frazzoli,
``Sampling-based algorithms for optimal motion planning,''
\textit{The International Journal of Robotics Research}, vol. 30, no. 7, pp. 846--894, 2011.

\bibitem{kavraki1998}
Lydia E. Kavraki, Mihail N. Kolountzakis, and J.-C. Latombe,
``Analysis of probabilistic roadmaps for path planning,''
\textit{IEEE Transactions on Robotics and Automation}, vol. 14, no. 1, pp. 166--171, 1998.

\bibitem{erke2020}
Shang Erke, Dai Bin, Yiming Nie, Qi Zhu, Liang Xiao, and Dawei Zhao,
``An improved A-Star based path planning algorithm for autonomous land vehicles,''
\textit{International Journal of Advanced Robotic Systems}, vol. 17, no. 5, 2020.

\bibitem{faust2018}
Aleksandra Faust, Kenneth Oslund, Oscar Ramirez, Anthony Francis, Lydia Tapia, Marek Fiser, and James Davidson,
``PRM-RL: Long-range robotic navigation tasks by combining reinforcement learning and sampling-based planning,''
in \textit{Proceedings of ICRA}, 2018, pp. 5113--5120.

\bibitem{pfeiffer2017}
Mark Pfeiffer, Michael Schaeuble, Juan Nieto, Roland Siegwart, and Cesar Cadena,
``From perception to decision: A data-driven approach to end-to-end motion planning for autonomous ground robots,''
in \textit{Proceedings of ICRA}, 2017, pp. 1527--1533.

\bibitem{zhang2025}
Dingqi Zhang, Antonio Loquercio, Jerry Tang, Ting-Hao Wang, Jitendra Malik, and Mark W Mueller,
``A learning-based quadcopter controller with extreme adaptation,''
\textit{IEEE Transactions on Robotics}, 2025.

\bibitem{ichter2018}
Brian Ichter, James Harrison, and Marco Pavone,
``Learning sampling distributions for robot motion planning,''
in \textit{Proceedings of ICRA}, 2018, pp. 7087--7094.

\bibitem{codevilla2018}
Felipe Codevilla, Matthias Müller, Antonio López, Vladlen Koltun, and Alexey Dosovitskiy,
``End-to-end driving via conditional imitation learning,''
in \textit{Proceedings of ICRA}, 2018, pp. 4693--4700.

\bibitem{manjunath}
Aniketh Manjunath and Quan Nguyen,
``Safe and robust motion planning for dynamic robotics via control barrier functions,''
in \textit{Proceedings of CDC}, 2021, pp. 2122--2128.

\bibitem{majumdar}
Anirudha Majumdar and Russ Tedrake,
``Funnel libraries for real-time robust feedback motion planning,''
\textit{The International Journal of Robotics Research}, vol. 36, no. 8, pp. 947--982, 2017.

\bibitem{gurgen}
Ali Ekin Gurgen, Anirudha Majumdar, and Sushant Veer,
``Learning provably robust motion planners using funnel libraries,''
arXiv preprint arXiv:2111.08733, 2021.

\bibitem{tsukamoto}
Hiroyasu Tsukamoto and Soon-Jo Chung,
``Learning-based robust motion planning with guaranteed stability: A contraction theory approach,''
\textit{IEEE Robotics and Automation Letters}, vol. 6, no. 4, pp. 6164--6171, 2021.

\bibitem{dhar2023}
Abhishek Dhar, Carl Hynén Ulfsjöö, Johan Löfberg, and Daniel Axehill,
``Disturbance-parametrized robust lattice-based motion planning,''
\textit{IEEE Transactions on Intelligent Vehicles}, vol. 9, no. 1, pp. 3034--3046, 2023.

\bibitem{bergman}
Kristoffer Bergman, Oskar Ljungqvist, Jonas Linder, and Daniel Axehill,
``An optimization-based motion planner for autonomous maneuvering of marine vessels in complex environments,''
in \textit{Proceedings of CDC}, 2020, pp. 5283--5290.

\bibitem{bergman1}
Kristoffer Bergman, Oskar Ljungqvist, and Daniel Axehill,
``Improved path planning by tightly combining lattice-based path planning and optimal control,''
\textit{IEEE Transactions on Intelligent Vehicles}, vol. 6, no. 1, pp. 57--66, 2020.

\bibitem{ljungqvist}
Oskar Ljungqvist, Daniel Axehill, and Johan Löfberg,
``On stability for state-lattice trajectory tracking control,''
in \textit{Proceedings of ACC}, 2018, pp. 5868--5875.

\bibitem{landau}
Ioan Doré Landau, Rogelio Lozano, Mohammed M'Saad, and Alireza Karimi,
\textit{Adaptive Control: Algorithms, Analysis and Applications},
Springer Science \& Business Media, 2011.

\bibitem{ioannou2006adaptive}
P. A. Ioannou and B. Fidan,
\textit{Adaptive Control Tutorial},
SIAM Press, 2006.

\bibitem{narendra2000adaptive}
Kumpati S. Narendra and Cheng Xiang,
``Adaptive control of discrete-time systems using multiple models,''
\textit{IEEE Transactions on Automatic Control}, vol. 45, no. 9, pp. 1669--1686, 2000.

\bibitem{narendra2011}
Kumpati S. Narendra and Zhuo Han,
``Discrete-time adaptive control using multiple models,''
in \textit{Proceedings of ACC}, 2011, pp. 2921--2926.

\bibitem{narendra2011mmac}
Kumpati S. Narendra and Zhuo Han,
``The changing face of adaptive control: the use of multiple models,''
\textit{Annual Reviews in Control}, vol. 35, no. 1, pp. 1--12, 2011.

\bibitem{khalil}
Hassan K. Khalil,
\textit{Nonlinear Systems},
Prentice-Hall, New Jersey, 1996.

\bibitem{han2011}
Zhuo Han and Kumpati S. Narendra,
``New concepts in adaptive control using multiple models,''
\textit{IEEE Transactions on Automatic Control}, vol. 57, no. 1, pp. 78--89, 2011.

\bibitem{dronedyna}
Haoping Wang, Xuefei Ye, Yang Tian, Gang Zheng, and Nicolai Christov,
``Model-free–based terminal SMC of quadrotor attitude and position,''
\textit{IEEE Transactions on Aerospace and Electronic Systems}, vol. 52, no. 5, pp. 2519--2528, 2016.
\end{thebibliography}
\end{document}